\documentclass[11pt]{extarticle}

\usepackage{arxiv}

\usepackage[utf8]{inputenc} % allow utf-8 input
\usepackage{hyperref}       % hyperlinks
\usepackage{url}            % simple URL typesetting
\usepackage{booktabs}       % professional-quality tables
\usepackage{amsfonts}       % blackboard math symbols
\usepackage{nicefrac}       % compact symbols for 1/2, etc.
\usepackage{microtype}      % microtypography
\usepackage{graphicx}
\usepackage[numbers]{natbib}
\usepackage{doi}
\usepackage{amsthm}
\usepackage{amsmath}
\usepackage{amssymb}
\usepackage{subfigure}
\newtheorem{theorem}{Theorem}[section]

\title{Optimizing an Adaptive Fuzzy Logic Controller of a 3-DOF Helicopter with a Modified PSO Algorithm}

\author{ Shokoufeh Naderi \\
	Département Génie Électrique \& Génie Informatique\\
	Université de Sherbrooke\\
	\texttt{shokoufeh.naderi@usherbrooke.ca} \\
	\And
	Maude J.~Blondin \\
	Département Génie Électrique \& Génie Informatique\\
	Université de Sherbrooke\\
	\texttt{maude.blondin2@usherbrooke.ca} \\
	\And
	Behrooz Rezaie \\
	Faculty of Electrical and Computer Engineering\\
	Babol Noshirvani University of Technology\\
	\texttt{brezaie@nit.ac.ir} \\
}

\date{}

\renewcommand{\shorttitle}{Optimizing an Adaptive Fuzzy Logic Controller of a 3-DOF Helicopter with a Modified PSO Algorithm}

\begin{document}
\maketitle

\begin{abstract}
	This paper investigates the controller optimization for a helicopter system with three degrees of freedom (3-DOF). To control the system, we combined fuzzy logic with adaptive control theory. The system is extensively nonlinear and highly sensitive to the controller's parameters,  making it a real challenge to study these parameters' effect on the controller's performance. Using metaheuristic algorithms for determining these parameters is a promising solution. This paper proposes using a modified particle swarm optimization (MPSO) algorithm to optimize the controller. The algorithm shows a high ability to perform the global search and find a reasonable search space. The algorithm modifies the search space of each particle based on its fitness function value and substitutes weak particles for new ones. These modifications have led to better accuracy and convergence rate. We prove the efficiency of the MPSO algorithm by comparing it with the standard PSO and six other well-known metaheuristic algorithms when optimizing the adaptive fuzzy logic controller of the 3-DOF helicopter. The proposed method's effectiveness is shown through computer simulations while the system is subject to uncertainties and disturbance. We demonstrate the method's superiority by comparing the results when the MPSO and the standard PSO optimize the controller.
\end{abstract}

% keywords can be removed

\section{Introduction}\label{sec1}
During the last decades, unmanned aerial vehicles (UAVs) have widely been developed and used due to technological advancements \cite{chen2018adaptive}. They have applications in military and civil fields, such as traffic condition assessment and forest fire monitoring, to name a few \cite{chen2015adaptive}. They possess essential features like hovering and vertical take-off, which increase their applicability. However, they are highly nonlinear and subject to disturbances and uncertainties. The non-linearity and susceptibility to disturbance demands for a control structure that can reject external disturbances, such as wind \cite{castaneda2016continuous}.

Many classic, adaptive, and robust control strategies have been proposed to tackle this control problem. In \cite{pounds2014stability}, for example, a proportional derivative (PD) and a proportional integral derivative (PID) attitude and position controllers are designed to stabilize a rotorcraft in free flight. 
In \cite{uddin2018active}, active vibration control is presented for a helicopter rotor blade that uses a linear quadratic regulator (LQR) to reduce vibrations. In \cite{dutta2021adaptive}, designing an adaptive model predictive control has been addressed for a 2-DOF helicopter in the presence of uncertainties and constraints. 

These control strategies perform well in the presence of parametric uncertainties. However, they may underperform in real-life applications with uncertainties, such as external disturbances, noises, and unmodeled dynamics of the machine, mainly because the methods are developed and based on the exact mathematical model of the system \cite{chaoui2020adaptive, chaoui2017adaptive}.

On the other hand, intelligent control techniques can adapt themselves in the presence of uncertainties. These approaches have various structures, including neural networks, fuzzy systems, and machine learning models \cite{salahshour2019designing}. The feature of not being dependent on a precise mathematical model has led to many publications on their combination with conventional control strategies for the UAVs. In \cite{gonzalez2018design}, the design and experimental validation of an adaptive fuzzy PID controller is presented for a 3-DOF helicopter. The interval type-2 fuzzy logic is combined with adaptive control theory to control a 3- DOF helicopter in \cite{chaoui2020adaptive}, which is robust to various types of uncertainties. However, using higher types of fuzzy systems increases the computational loads. Another example is the publication \cite{xue2019training} that proposes designing data-driven attitude controllers for a 3-DOF helicopter under multiple constraints, in which the reinforcement learning technique updates the controller. An adaptive neural network back-stepping controller is designed in \cite{yang2019adaptive} to compensate for unmodeled dynamics and external disturbances.

The problem with these intelligent controls is that they usually have many parameters to tune. Tuning is a complex task and almost impossible to do by trial and error. To overcome that, researchers usually implement metaheuristic algorithms \cite{naderi2020designing}. These algorithms have recently attracted an increased interest for this purpose thanks to their high convergence speed and high accuracy. For instance, in \cite{yu2019optimal} the PSO algorithm is implemented to optimize weighting matrices of the LQR controller to design an optimal flight control for a 2-DOF helicopter. The PSO is also used in \cite{humaidi2019particle} to tune the values of design parameters of an adaptive super-twisting sliding mode controller for a two-axis helicopter in the presence of model uncertainties. The publication \cite{tan2016genetic} employs a genetic approach for real-time identification and control of a helicopter, and \cite{ding2015chaotic} develops a chaotic artificial bee colony algorithm to identify a small-scale helicopter in hover condition. 

In some papers, the researchers have improved the performance of these algorithms by investigating the accuracy, speed, and convergence rate. In \cite{hu2020fuzzy}, for example, an improved genetic algorithm (GA) is used to optimize the initial fuzzy rules of an adaptive fuzzy PID controller for a micro-unmanned helicopter. However, in large-scale problems, where the search space is not clearly specified, the algorithm may get trapped into local optimal solutions \cite{salahshour2019designing}. Therefore, there is a need for algorithms that offers fast convergence and high accuracy while simultaneously being capable of modifying the search space.

In this paper, we employ a Modified Particle Swarm Optimization (MPSO) algorithm to optimize the parameters of an adaptive fuzzy controller for a 3-DOF helicopter. The helicopter system is highly nonlinear, and the controller's values significantly affect its performance. Therefore, optimizing these parameters can lead to more effective results. We chose the MPSO algorithm because it shows a satisfactory performance when dealing with many parameters to optimize. Furthermore, it eliminates the need for specifying the exact boundaries of the search space as the search space of the particles is modified based on the value of each particle. This results in searching in a reasonable space, which means not only does it shorten the optimization time, but it also results in finding a better solution by avoiding getting trapped into local optimums.
Additionally, the algorithm considers an elimination phase, meaning some poor particles are substituted by new particles in the new search space. These modifications have improved the convergence rate and accuracy of the algorithm. Indeed, simulation results in \cite{naderi2020designing} prove the superiority of this algorithm over several metaheuristic algorithms, including improved GA, imperialist competitive algorithms, and artificial bee colony. This paper shows the effectiveness of using the MPSO algorithm to optimize the controller's parameters for the 3-DOF helicopter through simulations when the system is subject to uncertainties and disturbance. Also, we compare the performance of the MPSO algorithm with the standard PSO algorithm and some other metaheuristic algorithms. 

In short, this paper contains the following contributions:

\begin{itemize}
    \item Optimizing an adaptive type-1 fuzzy logic controller for a 3-DOF helicopter model through an MPSO algorithm.
    \item Comparing the performance of the proposed controller with the PID controller and with the controller optimized through the PSO algorithm.
    \item Analyzing the robustness of the proposed controller in the presence of uncertainty and disturbance.
    \item Comparing the performance of the MPSO algorithm and some other well-known metaheuristic algorithms for the task of optimizing the adaptive type-1 fuzzy logic controller of the helicopter model along with analyzing their results.
\end{itemize}

The adaptive type-2 fuzzy logic controller for the 3-DOF helicopter is presented in \cite{chaoui2020adaptive}. Since the membership functions are fixed in the conventional type-1 fuzzy systems, which may lead to weak performance in the presence of uncertainties, the authors of \cite{chaoui2020adaptive} proposed using type-2 fuzzy controllers for the highly nonlinear helicopter system. However, they did not present the results of the type-1 fuzzy logic controller. In our paper, we use an adaptive type-1 fuzzy controller. We show that, for the considered task, the type-1 fuzzy logic controller performs very well under uncertainties and disturbance, eliminating the need for employing higher types of fuzzy systems. As a result, the computational load is lowered.

\begin{figure}[b!]%
\centering
\includegraphics[width=\textwidth]{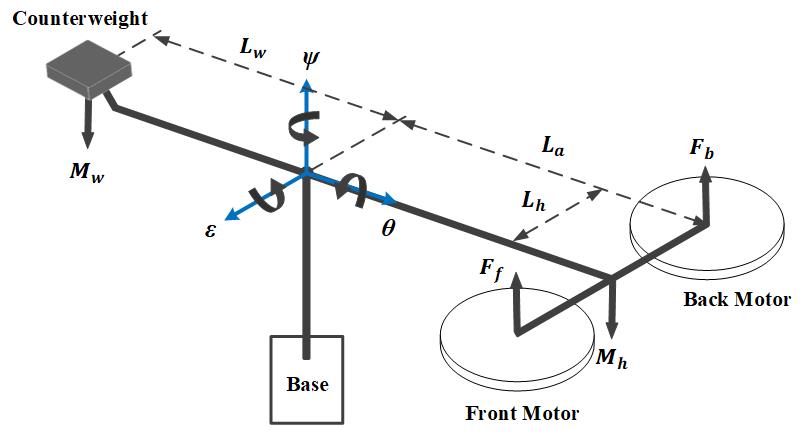}
\caption{The 3-DOF helicopter.}\label{Fig1}
\end{figure}

The rest of this paper is organized as follows:
\hyperref[sec2]{Sect. \ref{sec2}} presents the dynamic model of the helicopter. 
\hyperref[sec3]{Sect. \ref{sec3}} provides the adaptive fuzzy logic controller. 
\hyperref[sec4]{Sect. \ref{sec4}} describes the basic concepts of the MPSO algorithm. 
Simulation results and MPSO algorithm efficiency analysis are presented in \hyperref[sec5]{Sect. \ref{sec5}}, and \hyperref[sec6]{Sect. \ref{sec6}} provides concluding remarks.

\section{Mathematical model of the 3-DOF helicopter system}\label{sec2}

The schematic of the 3-DOF helicopter system is shown in \hyperref[Fig1]{Fig. \ref{Fig1}}. It has two DC motors mounted at the two ends of a frame. The DC motors drive two propellers which generate the lift forces $F_f$ and $F_b$. These forces control the attitude of the helicopter. The system studied here is underactuated, i.e., it rotates freely about three axes - the pitch axis $\theta$, the roll axis $\epsilon$, and the yaw axis $\psi$- with only two control forces. The mathematical model of a 3-DOF helicopter is given by the following differential equations  \cite{castaneda2016continuous}:
\begin{align} \label{eq.1}
&J_\epsilon \ddot{\epsilon}= \; g(M_h L_a-M_\omega L_\omega)cos\epsilon + L_a cos\theta \, u_1 \notag \\
&J_\theta \ddot{\theta}= \; L_h u_2\\
&J_\psi \ddot{\psi}= \; L_a cos\epsilon \, sin\theta \, u_1 \notag
\end{align}
where $ u_1 = F_f + F_b $, $ u_2 = F_f - F_b $. The pitch $\theta$, roll $\epsilon$, and yaw $\psi$ angles give the position of the helicopter body. Values of the model parameters with their description are given in \hyperref[table1]{Table \ref{table1}}. As it is shown in \hyperref[Fig1]{Fig. \ref{Fig1}}, the roll ($\epsilon$) and yaw ($\psi$) angles are perpendiculars. The intersection of the three axes is considered the origin of the coordinate frame. The helicopter model is defined as follows:
\begin{itemize}
\item[•] Pitch angle $\theta$ is defined as $-45^{\circ} \leq \theta \leq +45^{\circ}$.
\item[•] Roll angle $\epsilon$ is defined as $-27.5^{\circ} \leq \epsilon \leq +30^{\circ}$.
\end{itemize}
These constraints on the the pitch and roll angles will appear as saturation functions in the program.

\begin{table}[t!]
\begin{center}
%\begin{minipage}{174pt}
\caption{Values of the 3-DOF helicopter model \cite{castaneda2016continuous}.}\label{table1}%
\begin{tabular}{@{}p{3cm} p{7.8cm} p{2.2cm}@{}}
\toprule
Model's parameter &  Description & Value\\
\midrule
			$ M_h $  & Mass of the helicopter & $ 1.426 \: kg $ \\
			$ M_\omega $  & Mass of the counterweight & $ 1.870 \: kg $ \\
			$ L_a $  &  Distance between the roll axis and the center of mass & $ 0.660 \: m $ \\
			$ L_\omega $ &  Distance between the lifting axis and the counterweight & $ 0.470 \: m $ \\
			$ L_h $ & Distance between the pitch axis and each motor & $ 0.178 \: m $ \\
			$ J_\epsilon $ & Moment of inertia about roll axis & $ 1.0348\: kg.m^2 $ \\
			$ J_\theta $ & Moment of inertia about pitch axis & $ 0.0451 \: kg.m^2  $ \\
			$ J_\epsilon $ & Moment of inertia about yaw axis & $ 1.0348 \: kg.m^2 $ \\
			$ g $  & Gravitational constant & $ 9.81 \: m.s^2 $ \\
\bottomrule
\end{tabular}
%\end{minipage}
\end{center}
\end{table}

\section{Adaptive fuzzy logic controller}\label{sec3}
\subsection{Design of the controller}\label{subsec3.1}
As shown in \eqref{eq.1}, the roll and yaw dynamics depend on the pitch value and are actuated by $u_1$ while the pitch dynamic is actuated by $u_2$. This means that only two input forces control the system, i.e., the system is underactuated. To control the system with only two inputs,  we must first define the desired roll and yaw trajectories. Next, the desired pitch trajectory is obtained based on an internal loop. \hyperref[Fig3]{Fig. \ref{Fig3}} shows the block diagram of the control structure of the helicopter system. In the control structure, two virtual inputs $v_1$ and $v_2$ are defined to achieve decoupling  \cite{chaoui2020adaptive}:
\begin{align} \label{eq.4}
\begin{split}
&v_1 = \cos\theta \, u_1  \\
&v_2 = \cos\epsilon \, \sin\theta \, u_1
\end{split}
\end{align}

\begin{figure}[b!]%
\centering
\includegraphics[width=\textwidth]{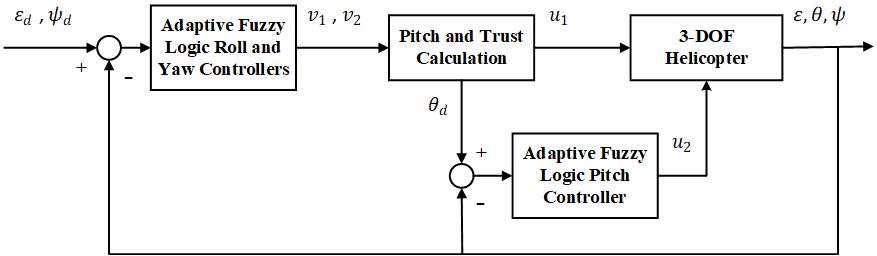}
\caption{Block diagram of the control structure.}\label{Fig3}
\end{figure}

Replacing the virtual inputs $v_1$ and $v_2$ in \eqref{eq.1} transforms the nonlinear system \eqref{eq.1} into the decoupled system \eqref{eq.5}, in which, $\epsilon$, $\theta$, and $\psi$, are controlled by $v_1$, $u_2$, and $v_2$, respectively:
\begin{align} \label{eq.5}
&J_\epsilon \ddot{\epsilon}= \; g(M_h L_a-M_\omega L_\omega)cos\epsilon + L_a v_1 \notag \\
&J_\theta \ddot{\theta}= \; L_h u_2\\
&J_\psi \ddot{\psi}= \; L_a v_2 \notag
\end{align}
Notice that $v_1$ and $v_2$ in \eqref{eq.4} are not independent but are linked through $u_1$.

\noindent From \eqref{eq.4} we have:
\begin{equation}\label{eq.6}
u_1^2 = \frac{v_1^2}{\cos^2\epsilon} + v_2^2
\end{equation}
Then, it follows that:
\begin{equation}\label{eq.7}
u_1 = S\sqrt{\frac{v_1^2}{\cos^2\epsilon} + v_2^2}\;, \quad S = \begin{cases}
\text{sign}(v_2) &  v_2 \ne 0\\
0 &  v_2 = 0\\ 
\end{cases}
\end{equation}
\eqref{eq.4} also satisfies:
\begin{equation}\label{eq.8}
\tan \theta = \frac{v_2}{\cos\epsilon v_1}
\end{equation}
As we said, a way to control the system by only two inputs is to first define the desired roll and yaw trajectories, and then, to obtain the desired pitch trajectory. From \eqref{eq.8}, this desired pitch trajectory is obtained as follows:
\begin{equation}\label{eq.9}
\theta_d = \tan^{-1} \frac{v_2}{\cos\epsilon v_1}
\end{equation}
So the control strategy is as follows: 
\\
First, the control inputs $v_1$ and $v_2$ are computed using the tracking error between $\epsilon$ and $\psi$ and their desired trajectories $\epsilon_d$ and $\psi_d$. Then, the pitch desired trajectory $\theta_d$ and $u_1$ are computed using \eqref{eq.9} and \eqref{eq.7}. Once the $\theta_d$ is computed, the pitch controller allows to design $u_2$.

Given the desired trajectories, the errors are defined as follows:
\begin{align}\label{eq.10}
&e_\epsilon = \epsilon - \epsilon_d \notag\\
&e_\theta = \theta - \theta_d\\
&e_\psi = \psi - \psi_d \notag
\end{align}

Each adaptive fuzzy logic controller has to adjust its weight to track these errors to zero. The errors and their derivatives are considered as inputs of the controllers. So, each adaptive fuzzy controller has two inputs and an output. Using the centroid defuzzification method, the output of the controller is as follows:
\begin{equation}\label{eq.11}
Y(x) = \frac{ \sideset{}{_{i=1}^{n}}{\sum}f^i y^i} { \sideset{}{_{i=1}^{n}}{\sum}f^i},
\end{equation}
where $n$ is the number of fuzzy logic rules, and $y^i$ is the output singleton number. $f^i$ is expressed by the following equation:
\begin{equation}\label{eq.12}
f^i = \prod_{j=1}^{n} \mu_{{A}_j^i} (x_j)
\end{equation}
where $\mu_{{A}_j^i}$ is the value of the membership function for the fuzzy variable $x_j$.

Considering $W \in \mathbb{R}^n$ as the adjustable parameter vector composed of the fuzzy logic consequent part, the output of the adaptive fuzzy logic controller is expressed by:
\begin{equation}\label{eq.13}
Y = \Phi^T + \sigma = \hat \Phi^T \hat W,
\end{equation}
where $\sigma$ is the fuzzy logic output error, and $ \hat \Phi \in \mathbb{R}^n $ is an $n$-dimensional vector representing known functions of the fuzzy logic antecedent part \cite{chaoui2020adaptive}. $ \hat \Phi $ is defined as follows:
\begin{equation}\label{eq.14}
\hat \Phi = \frac{f^i} { \sideset{}{_{i=1}^{n}}{\sum}f^i}.
\end{equation}
More detail about the choice of fuzzy rules can be found in \cite{chaoui2017adaptive}.

The control law is given as \cite{chaoui2020adaptive}:
\begin{align}\label{eq.15}
&v_1 = \hat \Phi^T_\epsilon \hat W_\epsilon \notag \\
&v_2 = \hat \Phi^T_\psi \hat W_\psi \\
&u_2 = \hat \Phi^T_\theta \hat W_\theta \notag 
\end{align}

The controllers have to drive the tracking errors $e_\epsilon$, $e_\psi$, and $e_\theta$ to zero. To guarantee that the outputs of the system track the desired trajectories, we must have a tuning method for the parameters $\hat W$. The following theorem provides this tuning method \cite{teiar2014adaptive, chaoui2020adaptive}.

\begin{theorem}
Considering a nonlinear system in the form of \eqref{eq.1}, \eqref{eq.4}, and \eqref{eq.5} with the control law \eqref{eq.15}, the stability of the closed-loop system is achieved with the following adaptation law:
\begin{equation}\label{eq.16}
\dot{\hat W} = -\Gamma \hat \Phi B^T P E,
\end{equation}
where $\Gamma = diag(\gamma_1, \gamma_2, ... \, , \gamma_j)$ and $\gamma_l$ is a positive constant, $l=1,...\, ,j$. $P$ is a chosen symmetric positive definite matrix that satisfies the following Lyapunov equation:
\begin{equation}\label{eq.17}
A^T P+PA=-Q,
\end{equation}
with $Q>0$, and:
\begin{equation*}
E=
\begin{bmatrix}
e\\
\dot e
\end{bmatrix}, 
\; A=
\begin{bmatrix}
0 &1\\
-K_p &-K_d
\end{bmatrix}, 
\;B=
\begin{bmatrix}
0\\
\hat \eta
\end{bmatrix}, 
\end{equation*}
where $K_p>0$, $K_d>0$, and $\hat \eta$ is the estimation of $\eta$, defined as:
\begin{equation}\label{eq.18}
\eta_\epsilon = \frac{L_a}{J_\epsilon}, \quad 
\eta_\psi = \frac{L_a}{J_\psi}, \quad 
\eta_\theta = \frac{L_h}{J_\theta}
\end{equation}
\end{theorem}

\begin{proof}
See \cite{chaoui2020adaptive}.
\end{proof}

\subsection{Optimization of the controller}\label{subsec3.2}
During the adaptive fuzzy controller design, one issue arises, i.e., selecting the design parameter values of the controller, including $K_p, K_d$, and the parameters of the fuzzy system. For each fuzzy controller, Gaussian membership functions have been considered. Each membership function has two parameters, i.e., center $c$ and sigma $\sigma$. Choosing the right values for these parameters can greatly affect the performance of the controllers. Moreover, because the 3-DOF helicopter model is highly complex and nonlinear, any small change in the $K_p$ and $K_d$ values can lead to unstable behavior. In this paper, the MPSO algorithm is implemented to find the optimal values of these parameters. A fitness function, i.e., a function computes the cost of each particle while the algorithm searches for the best value. We choose the root mean square error (RMSE) to be the cost function because it intensifies the impact of large errors:
\begin{equation}\label{eq.19}
RMSE = \sqrt{\frac{1}{n} \sideset{}{_{i=1}^{n}}{\sum} {(e_\epsilon^2 + e_\psi^2 + e_\theta^2)} }
\end{equation}

\section{MPSO algorithm}\label{sec4}
The PSO algorithm is an optimization algorithm based on the behavior of bird flocking. % and was first introduced in 1995 \cite{kennedy1995particle}. 
For the first iteration of the algorithm, the velocity and position of all particles are initialized randomly. They share their information with their neighbors as they move towards the solution. Therefore, apart from the experience that each particle gains, they also use other particles' experience. This combination leads to updating their positions and velocities based on their knowledge and the most successful particle's experience, i.e., the one that is closer to the target. And a cost function measures this success. The longer the distance between a particle and the target is, the higher the cost function value for that particle is.  

The following equations update the position and velocity of particle $i$ \cite{naderi2020designing}:
\begin{align}\label{eq.2&3}
v_i(t+1) =\; & w v_i(t)  + c_1 r_1 ( p_{best,i}(t) - x_i(t) )\\ \notag
& + c_2 r_2 (g_{best}(t) - x_i(t)),\\
x_i(t+1) =\; & x_i(t)  +  v_i(t+1), 
\end{align}
where $ v_i(t) $ and $ x_i(t) $ are, respectively, the velocity and the position of the $i^{th}$ particle at  instance $t$. $ c_1 $ and $ c_2 $ are cognitive and social acceleration factors, respectively. $ r_1 $ and $ r_2 $ are uniform random numbers distributed within the interval [0, 1], and  $ w $ is the inertia weight. $ p_{best,i} $ and $ g_{best} $ are, respectively, the best solution that the $i^{th}$ particle and all particles have obtained so far. In every iteration, the best solution of each particle and the best solution of all particles are stored. In the following iteration, the algorithm uses the stored information to modify the position and velocity of particles. Hence, the particles gradually move towards $ g_{best} $ until reaching the target. The parameters of \eqref{eq.2&3} are selected as follows:
\begin{itemize}
\item[•] The inertia weight ($ w $) balances the local and global search. A large value intensifies the global search and a small value the local search. In the beginning, its value should be large to search the entire space extensively. Then, it should gradually be reduced to reach the optimal solution.
\item[•] The cognitive and social acceleration factors ($ c_1 $ and $ c_2 $) also help balance the local and global search; they are often set to the same value. They represent the effects of $ p_{best,i} $ and $ g_{best} $ on the velocity of each particle.
\end{itemize}

The iterative procedure of the algorithm continues until the stopping condition is met, such as achieving an acceptable cost function value or reaching a maximum number of iterations.

The PSO algorithm is easy to implement with few parameters to modify. Nevertheless, there are some drawbacks. In problems with many local minimums, the algorithm is susceptible to falling into one. Moreover, specifying the exact area of the search space is not always possible. To overcome these drawbacks, a modified version of this algorithm is presented in \cite{naderi2020designing}.

\begin{figure}[t!]%
\centering
\includegraphics[width=\textwidth]{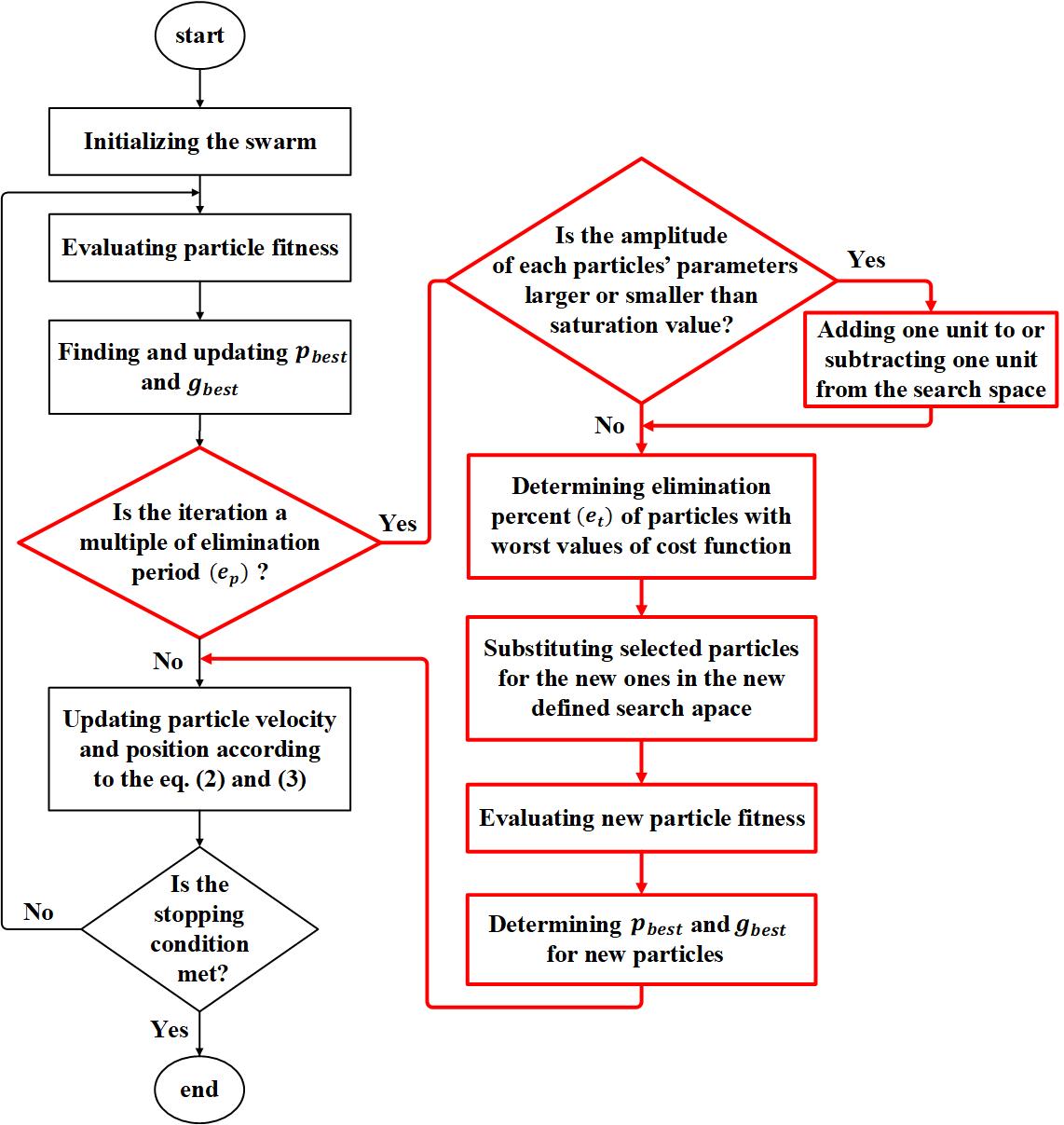}
\caption{Flowchart of the MPSO algorithm (red blocks show the parts added to the standard PSO algorithm).}\label{Fig2}
\end{figure}

In the modified version, the search space of each particle is modified based on the parameters of that particle. So after some iterations, the search space of each particle is decreased or increased independent of other particles. An adaptive search space leads to searching in a more reasonable bounded space. Furthermore, the algorithm considers an elimination phase; the new particles substitute for the poor particles. i.e., particles with the largest cost functions. This elimination phase, $e_p$, depends on the number of iterations. For example, $e_p$ set to 30 means that the elimination phase is performed every 30 iterations. \hyperref[Fig2]{Fig. \ref{Fig2}} presents the flowchart of the MPSO algorithm. The red blocks show the parts added to the standard PSO algorithm to clarify the differences. The $e_t$ in the flowchart is determined based on the percentage of the initial population and shows the number of particles to be deleted in every elimination phase. For example, if $e_t = 20$, 20\% of all particles with highest cost function values are substituted by new particles in the new search space at every elimination step. The percentage of saturation is an arbitrary value of the upper or lower bound of the search space \cite{salahshour2019designing}. For instance, if the initial bounds of the search space are considered to be -1 and 1, and the saturation value is set to 90\%, the values of each particle's parameters are checked in each elimination phase. Then, for the upper bound, 1, if these values are less than or greater than 1, the boundary is decreased to 0.9 or increased to 1.1. For the lower bound, -1, if the values are lower than or greater than -1, the boundary is decreased to -1.1 or increased to -0.9. So, when the value is lower than the bound, one unit is subtracted from the search space, and when it is greater than the bound, one unit is added. This unit is determined based on the saturation value: 100\% - 90\% (saturation value) $=$ 0.1.

\section{Simulation result}\label{sec5}
This section presents the effectiveness of the proposed adaptive fuzzy controller. We compare the performance of the MPSO and the standard PSO algorithms to optimize the considered controller for the 3-DOF helicopter. A comparison has also been drawn between the adaptive fuzzy controller and a classical PID controller.

\hyperref[table1]{Table \ref{table1}} shows the values of the 3-DOF helicopter model. The initial values of the pitch, roll, and yaw angles are set to zero.  The following equation is considered to be the desired trajectory for the roll and yaw angles \cite{chaoui2020adaptive}:
\begin{equation}\label{eq.20}
\epsilon_d , \psi_d = \frac 1 {1+e^{-2.5(t+2)}}
\end{equation}

As stated in \hyperref[sec3]{Sect. \ref{sec3}}, the value of the inertia weight $w$ should be large at the beginning to help the global search and then be reduced to find the optimal solution. Therefore, $w=1$ at the beginning, and linearly is reduced to 0.98 of its value at each iteration. The acceleration factors, $c_1$ and $c_2$, are both set to 2, and the population size is chosen to be 30. These settings are the same for both PSO and MPSO algorithms. In the MPSO algorithm, the elimination percent $e_t$ and the elimination period $e_p$ are set to 75 and 40, respectively. There are two inputs $e$ and $\dot e$, and seven Gaussian membership functions for each input, named Negative Large (NL), Negative Medium (NM), Negative Small (NS), Zero (Z), Positive Small (PS), Positive Medium (PM), and Positive Large (PL) \cite{teiar2014adaptive}. Both algorithms have been run 25 times, and the best results for each algorithm have been reported.

\begin{table}[t!]
\begin{center}
%\begin{minipage}{174pt}
\caption{Parameter values of the optimized controllers.}\label{table2}%
\begin{tabular}{@{}p{2.5cm} p{4.5cm} p{5.5cm}@{}}
\toprule
Parameters & MPSO & PSO\\
\midrule
			$ K_\epsilon $ & 83.21 & 68.46  \\
			$ K_\psi $ & 168.53 & 157.95  \\
			$ K_\theta $ & 10.15 & 125.58 \\
			$ K_p $ & 1.78 & 5.57  \\
			$ K_d $ & 48.46 & 33.19  \\
			$ \Gamma_\epsilon = \Gamma_\psi $ & diag\,[79 68 33  0 79 42 76] & diag\,[70 102 77 125 55 13 112] \\
			$ \Gamma_\theta $ & diag\,[53 48 11 49 3 88 19] & diag\,[172 100 119 185 70 155 113] \\
			$ {\text c_{NL}}_{e} $ & -2.80 & -120.92 \\
			$ {\text c_{NM}}_e $ & -0.00 & -48.60  \\
			$ {\text c_{NS}}_{e} $ & -0.00 & -0.04  \\
			$ {\text c_{Z}}_{e} $ & 0.00 & 0.00  \\
			$ {\text c_{PS}}_{e} $ & $ {-\text c_{NS}}_{e} $ & $ {-\text c_{NS}}_{e} $ \\
			$ {\text c_{PM}}_{e} $ & $ {-\text c_{NM}}_{e} $ & $ {-\text c_{NM}}_{e} $ \\
			$ {\text c_{PL}}_{e} $ & $ {-\text c_{NM}}_{e} $ & $ {-\text c_{NM}}_{e} $ \\
			$ {\text c_{NL}}_{\dot e} $ & -9.51 & -129.19 \\
			$ {\text c_{NM}}_{\dot e} $ & -6.32 & -125.02  \\
			$ {\text c_{NS}}_{\dot e} $ & -0.97 & -0.75 \\
			$ {\text c_{Z}}_{\dot e} $ & 0.00 & 0.00  \\
			$ {\text c_{PS}}_{\dot e} $ & $ {-\text c_{NS}}_{\dot e} $ & $ {-\text c_{NS}}_{\dot e} $  \\
			$ {\text c_{PM}}_{\dot e} $ & $ {-\text c_{NM}}_{\dot e} $ & $ {-\text c_{NM}}_{\dot e} $  \\
			$ {\text c_{PL}}_{\dot e} $ & $ {-\text c_{NM}}_{\dot e} $ & $ {-\text c_{NM}}_{\dot e} $ \\
			RMSE & 0.187 & 0.236 \\
\bottomrule
\end{tabular}
%\end{minipage}
\end{center}
\end{table}

\hyperref[table2]{Table \ref{table2}} represents the final optimized values of controllers' parameters with the cost function values. The centers $c$ of membership functions for all three controllers are the same, and sigmas are considered to be 1. As the table shows, the centers of membership functions have substantially high values when optimized with the PSO algorithm. This happens when we already know the approximate boundary of some parameters, but we cannot define it in the algorithm. Indeed, defining the upper and lower bound of the search space for each parameter is possible in the MPSO algorithm, allowing it to avoid falling into local minimums.

\hyperref[Fig4]{Fig. \ref{Fig4}}, \hyperref[Fig4]{Fig. \ref{Fig4.2}}, and \hyperref[Fig4]{Fig. \ref{Fig4.3}} present the yaw, pitch, and roll angles, the tracking errors, and the control signals for the nominal case. In these figures, the adaptive fuzzy logic controller's performance has been compared with a classical PID controller's performance. The figures reveal a steady-state error of less than 0.02 for the roll and yaw angles with the PID controller. The system's settling time (for all three angles) is higher for the PID controller. The adaptive fuzzy logic controller improves tracking when optimized with PSO and MPSO algorithms. The PID controller's only priority over the adaptive fuzzy controller optimized by the PSO algorithm is its lower control effort. Overall, these figures prove the superiority of the adaptive fuzzy logic controller compared to the PID controller.

\begin{figure}[t!]%
\centering
		\subfigure[]{
			\includegraphics[width=7cm]{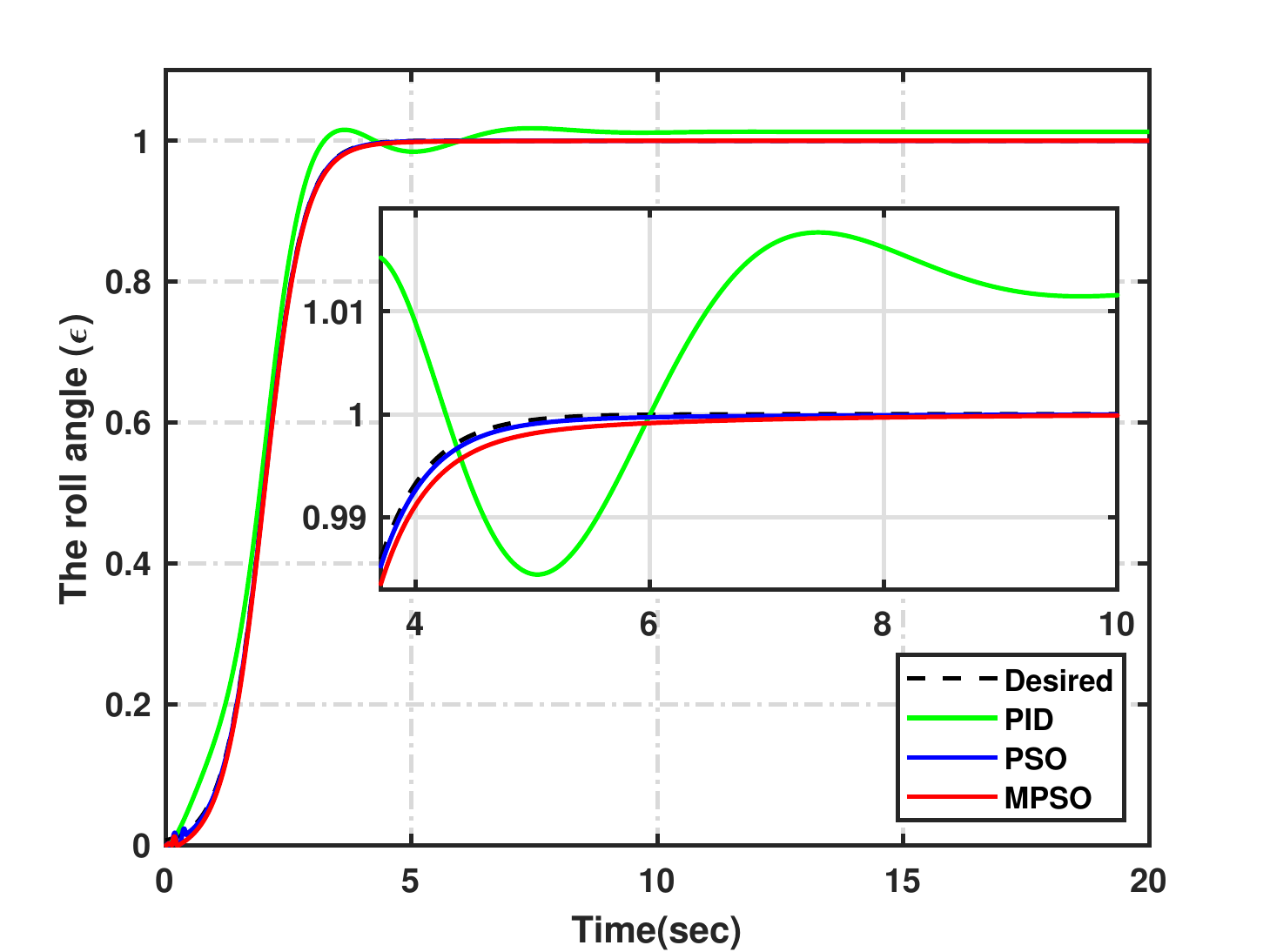}
			\label{Fig4a}	
		}
		\subfigure[]{
			\includegraphics[width=7cm]{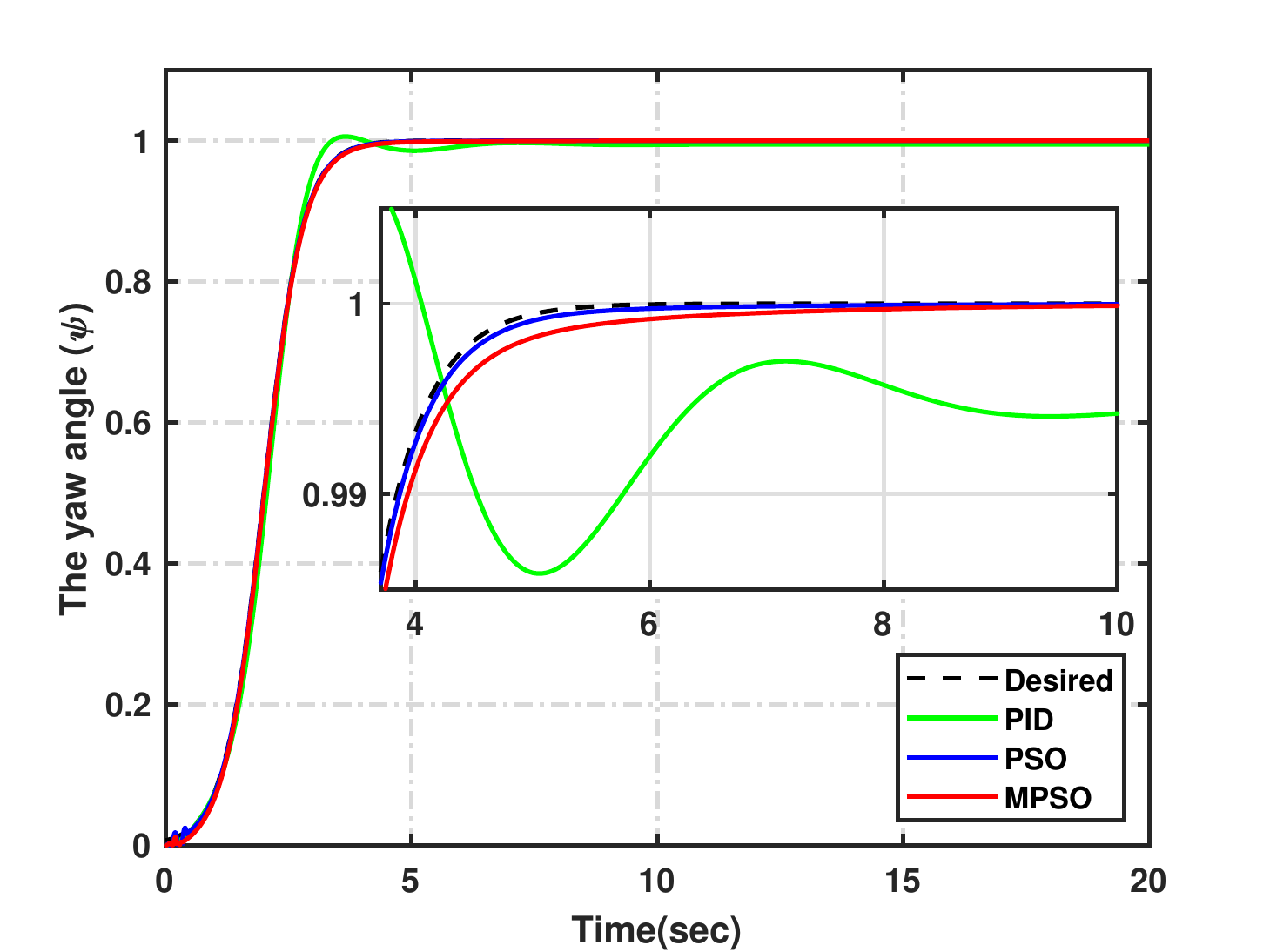}
			\label{Fig4b}
		}
		\subfigure[]{
			\includegraphics[width=7cm]{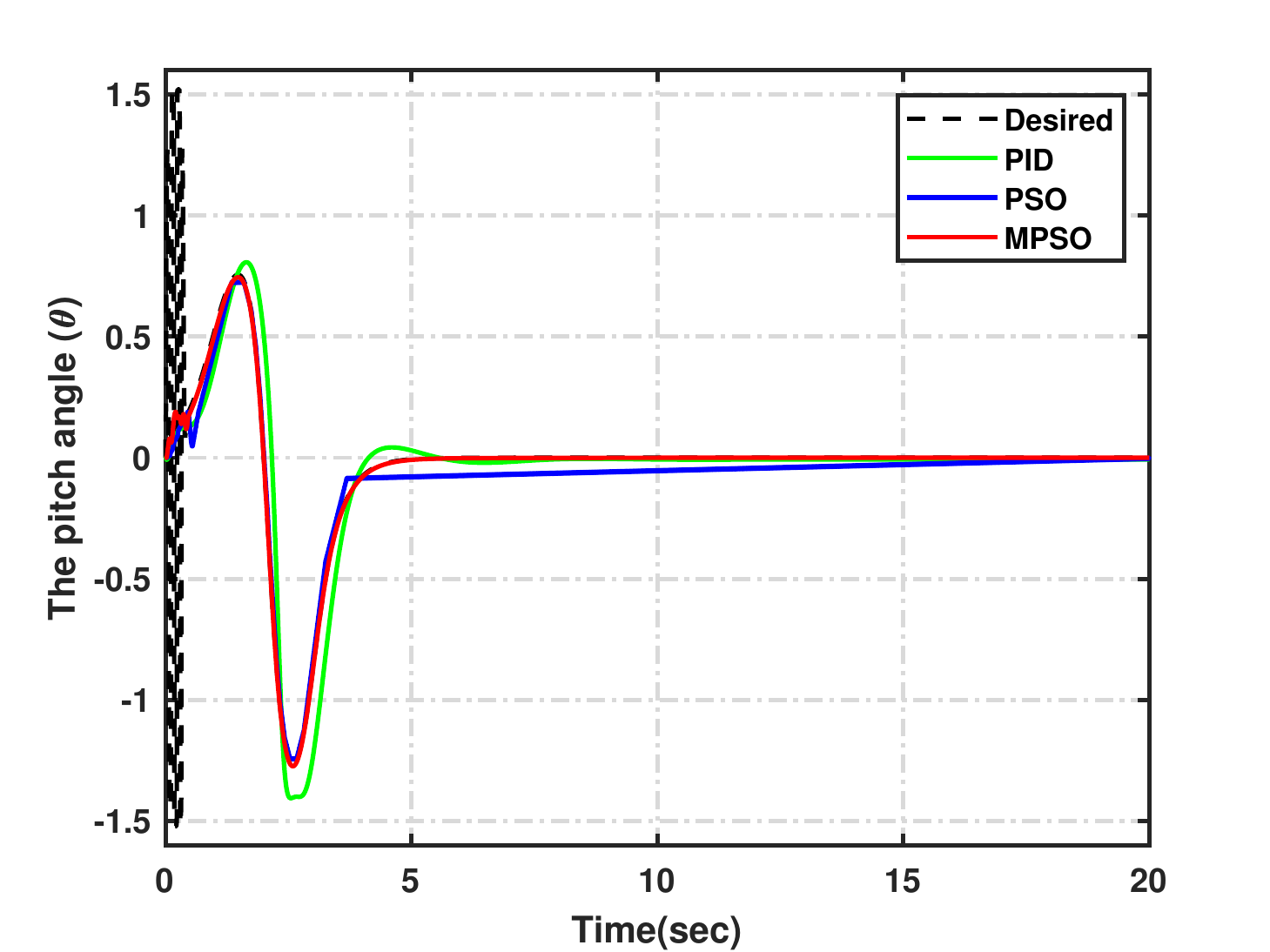}
			\label{Fig4c}
		}
\caption{Simulation results in the nominal case: outputs of the system.}\label{Fig4}
\end{figure}
\begin{figure}[t!]%
\centering
		\subfigure[]{
			\includegraphics[width=7cm]{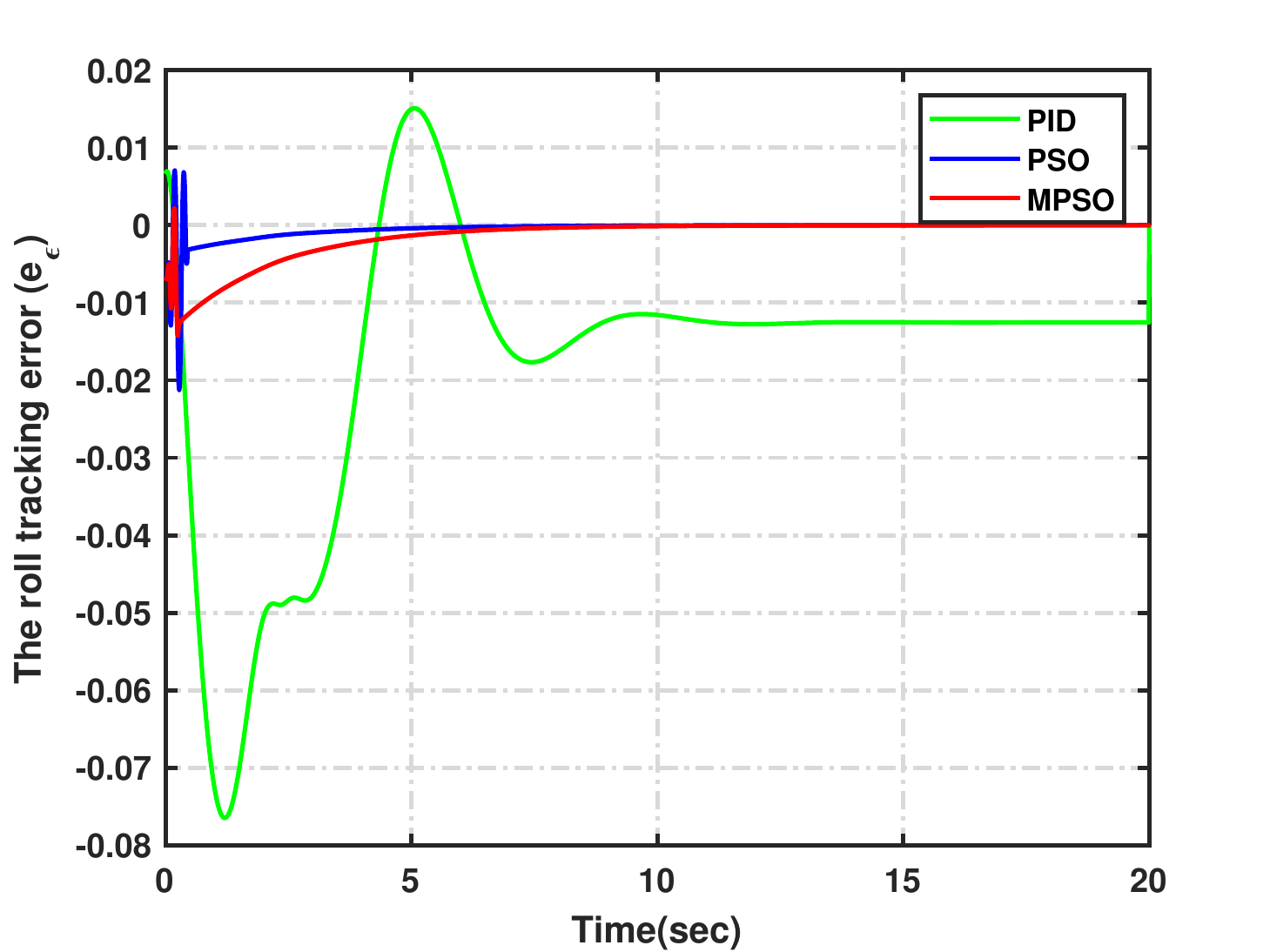}
			\label{Fig4d}	
		}
		\subfigure[]{
			\includegraphics[width=7cm]{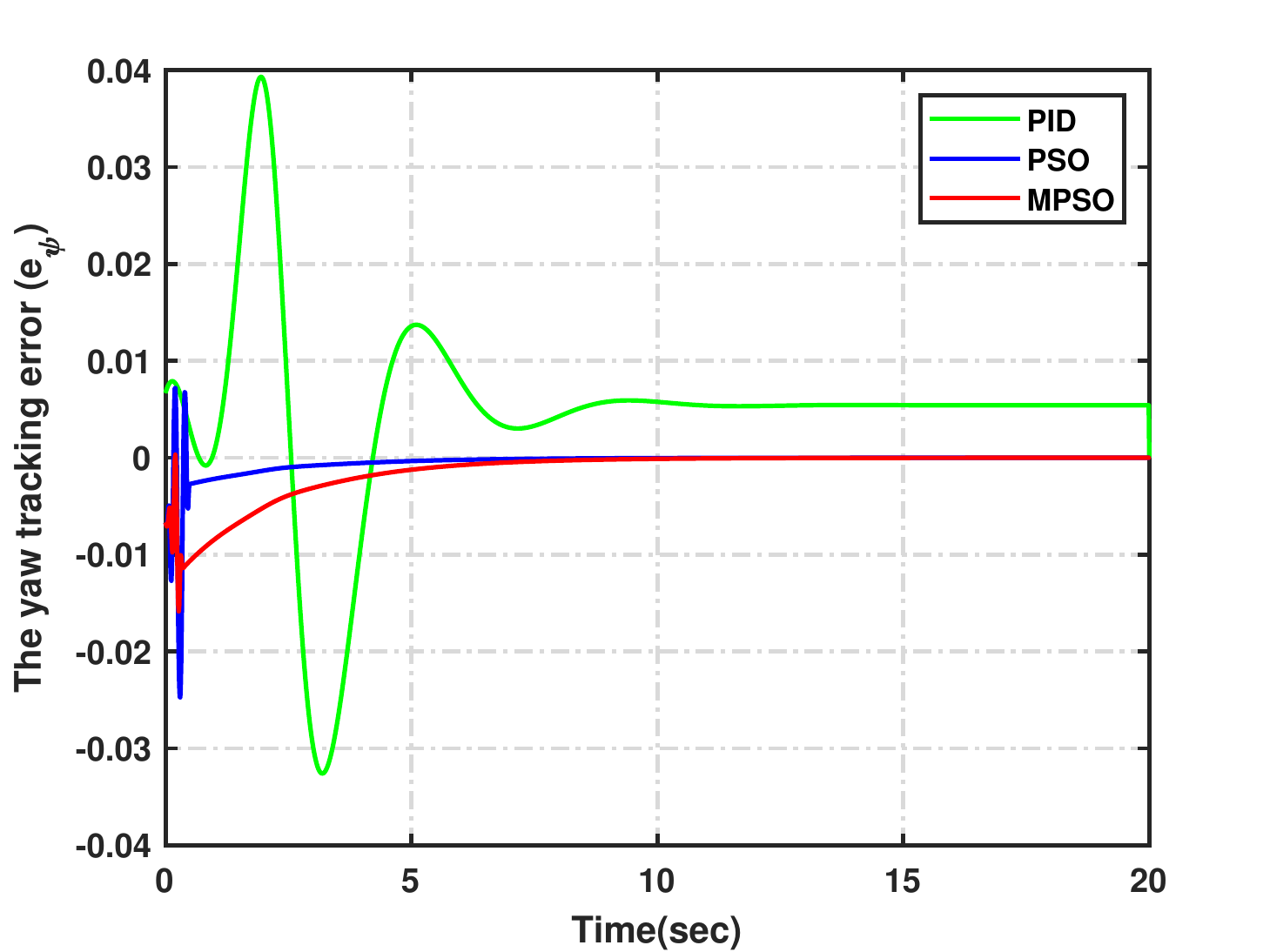}
			\label{Fig4e}
		}
		\subfigure[]{
			\includegraphics[width=7cm]{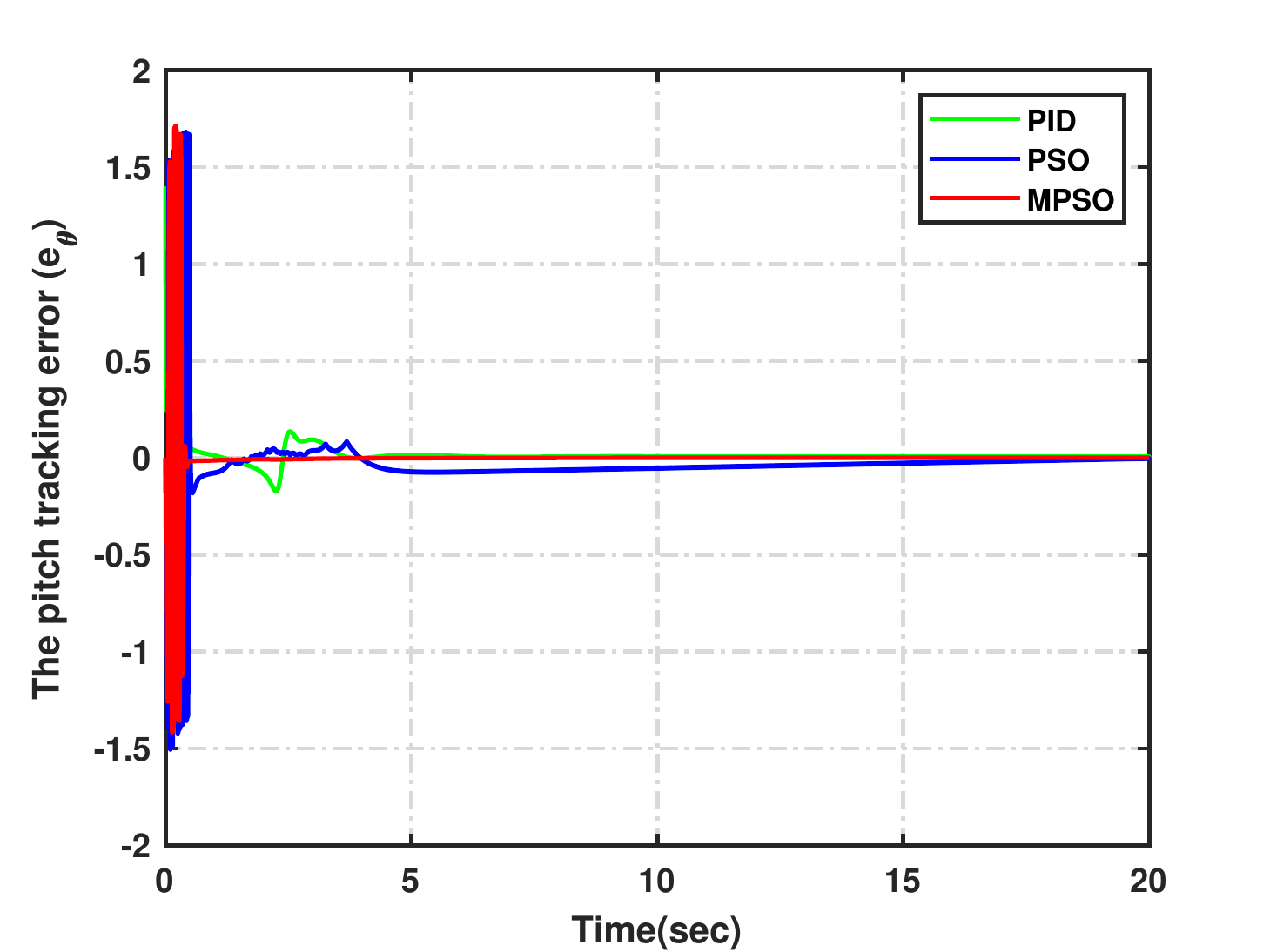}
			\label{Fig4f}
		}
\caption{Simulation results in the nominal case: motion tracking errors.}\label{Fig4.2}
\end{figure}
\begin{figure}[t!]%
\centering
		\subfigure[]{
			\includegraphics[width=7cm]{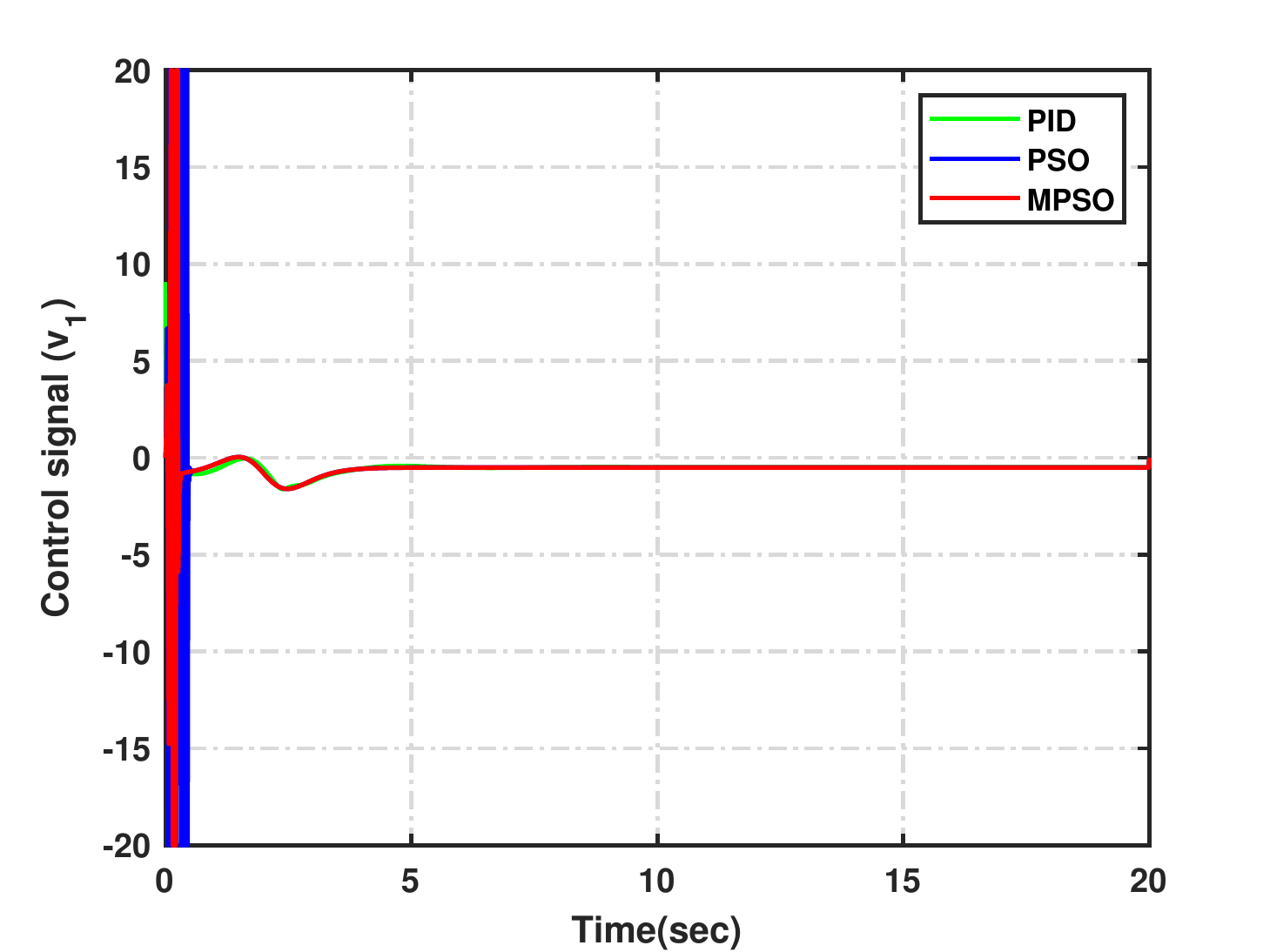}
			\label{Fig4g}	
		}
		\subfigure[]{
			\includegraphics[width=7cm]{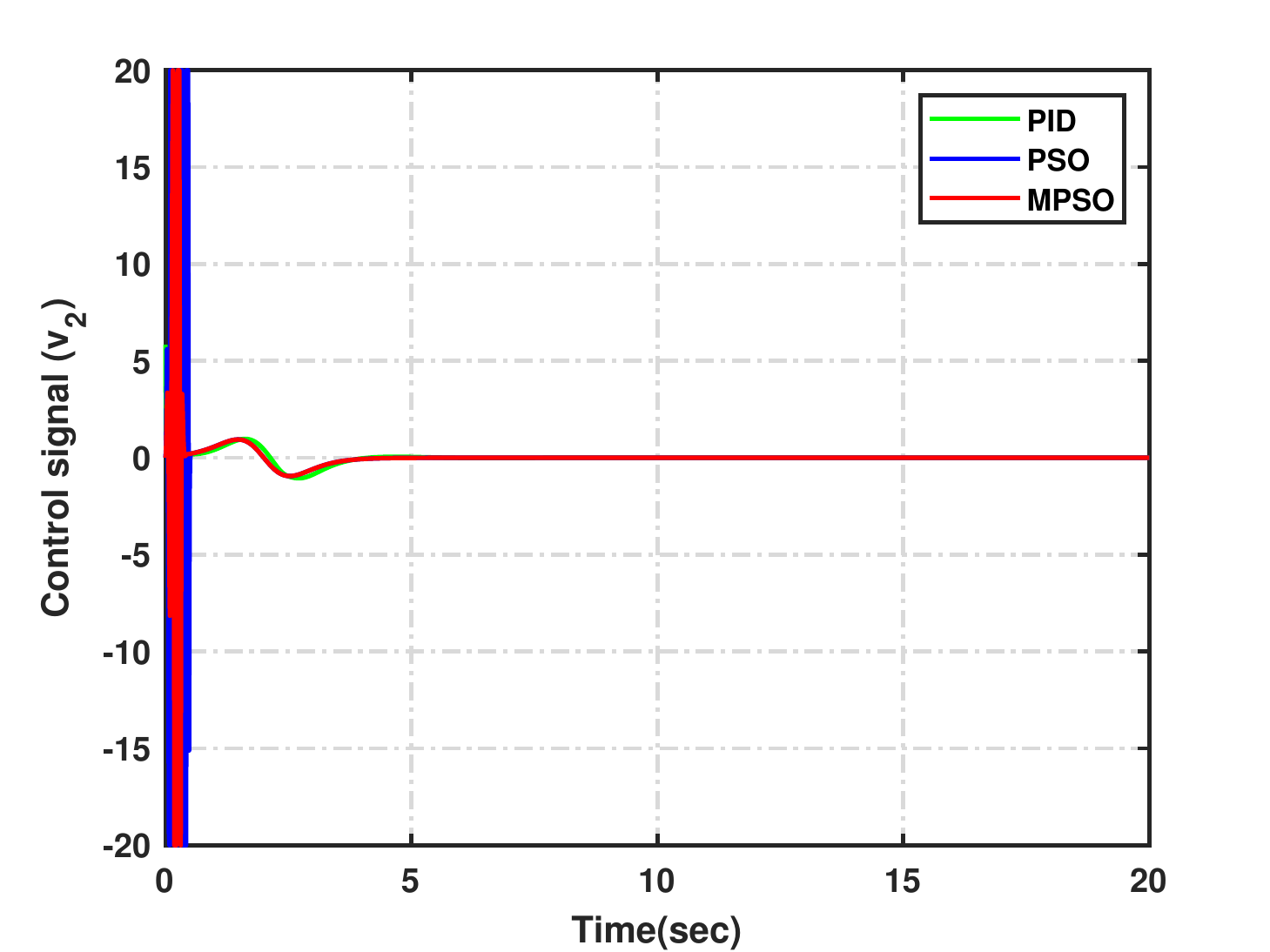}
			\label{Fig4h}
		}
		\subfigure[]{
			\includegraphics[width=7cm]{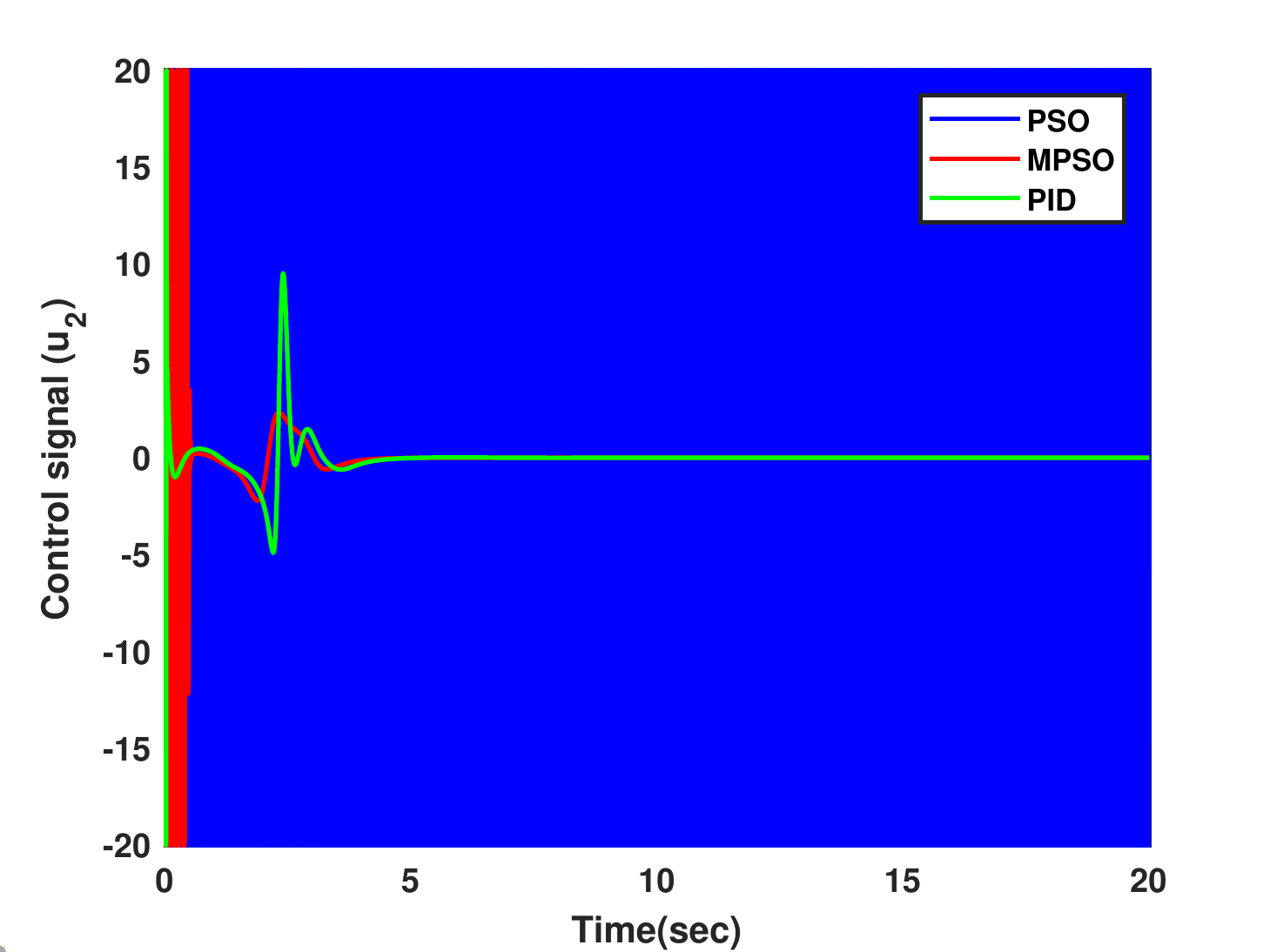}
			\label{Fig4i}
		}
\caption{Simulation results in the nominal case: control signals.}\label{Fig4.3}
\end{figure}

Comparing the performance of PSO and MPSO algorithms on \hyperref[Fig4]{Fig. \ref{Fig4}}, \hyperref[Fig4]{Fig. \ref{Fig4.2}}, and \hyperref[Fig4]{Fig. \ref{Fig4.3}}, it can be seen that their results are almost the same for the roll and yaw angles, with the PSO algorithm having less settling time. But the difference increases for the pitch angle, where the MPSO algorithm has much less tracking error.
Consequently, the control signals for the PSO algorithm fluctuate more, especially for the pitch angle. The main reason for this fluctuation in the control signal is the incapability of the PSO algorithm in defining separate search spaces for different parameters and modifying these search spaces based on the value of each parameter. In other words, it easily gets trapped into local minimums when there are many local minimums or many parameters to optimize.

To further investigate the algorithms' performance, the results are compared in \hyperref[table3]{Table \ref{table3}} based on the integral of the absolute values of control signals (IACS), which is calculated as follows:
\begin{equation} \label{eq.21}
IACS = \int_0^T (\vert v_1\vert + \vert v_2\vert + \vert u_2\vert ) \,dt ,
\end{equation}
where $v_1$, $v_2$, and $u_2$ are the roll, yaw, and pitch control signals, and $T$ is the final time instant. Notice that IACS is only a metric to compare the performance of algorithms and not the objective function. It has been computed after optimizing the parameters of controllers and applying them to the model. As the results show, the MPSO has a lower and 
IACS value, and therefore, better performance.

\begin{table}[h!]
\begin{center}
%\begin{minipage}{174pt}
\caption{Comparison between the performance of PSO and MPSO algorithms.}\label{table3}%
\begin{tabular}{@{}p{3cm} p{2cm} @{}}
\toprule
 & IACS \\
\midrule
			PSO & 213.429 \\
			MPSO & 15.177\\
\bottomrule
\end{tabular}
%\end{minipage}
\end{center}
\end{table}

To demonstrate the robustness of the controllers, we tested them under parametric uncertainties. For this purpose, we decreased the system's mass to half of its nominal case, and in another case, we increased it to one and a half times the nominal case, as it was done in \cite{chaoui2020adaptive}.  \hyperref[Fig5]{Fig. \ref{Fig5}} to \hyperref[Fig6]{Fig. \ref{Fig6.3}} depict the results. The roll angle for the PSO algorithm fluctuates more than the nominal case, leading to more fluctuation in the control signal. Moreover, the pitch angle for the PSO algorithm does not track the desired trajectory. But for the MPSO algorithm, results are identical to the nominal case, demonstrating the controller's robustness facing uncertainties. The only difference is the magnitude of the $u_2$ to adjust the mass change. It decreases when the system's mass is half of its nominal value and increases when the mass has increased one and a half times. 

\begin{figure}[t!]%
\centering
		\subfigure[]{
			\includegraphics[width=7cm]{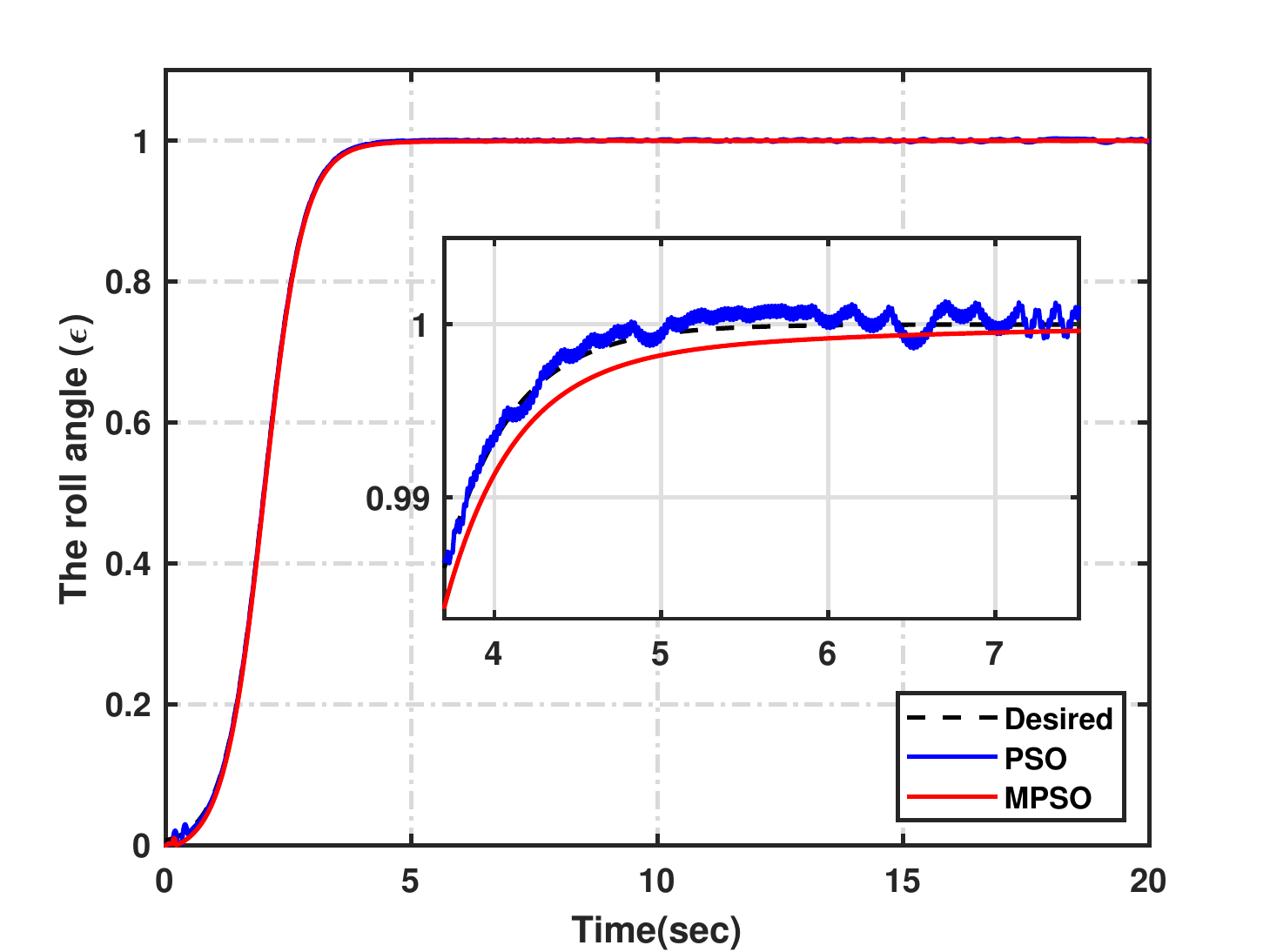}
			\label{Fig5a}	
		}
		\subfigure[]{
			\includegraphics[width=7cm]{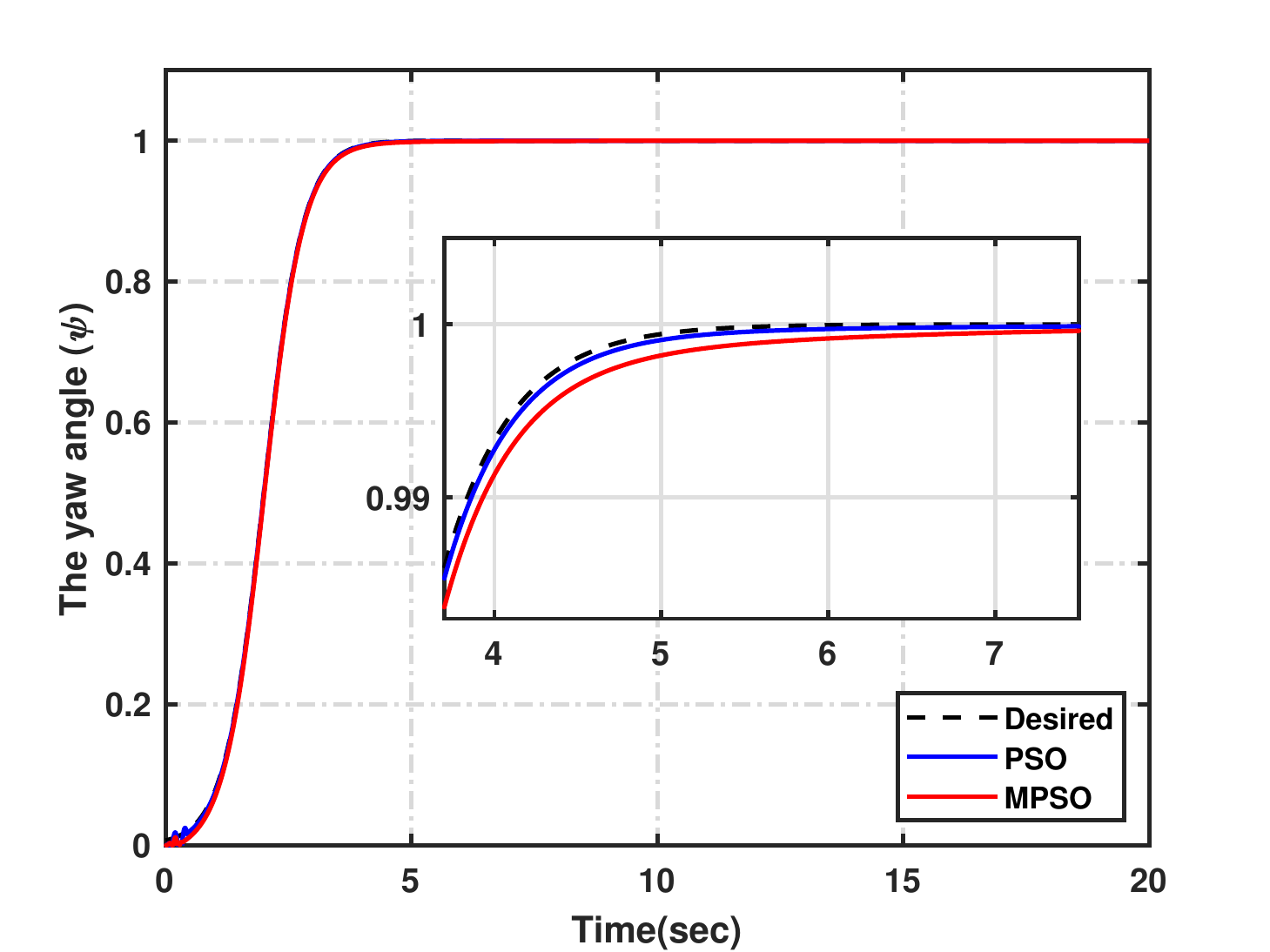}
			\label{Fig5b}
		}
		\subfigure[]{
			\includegraphics[width=7cm]{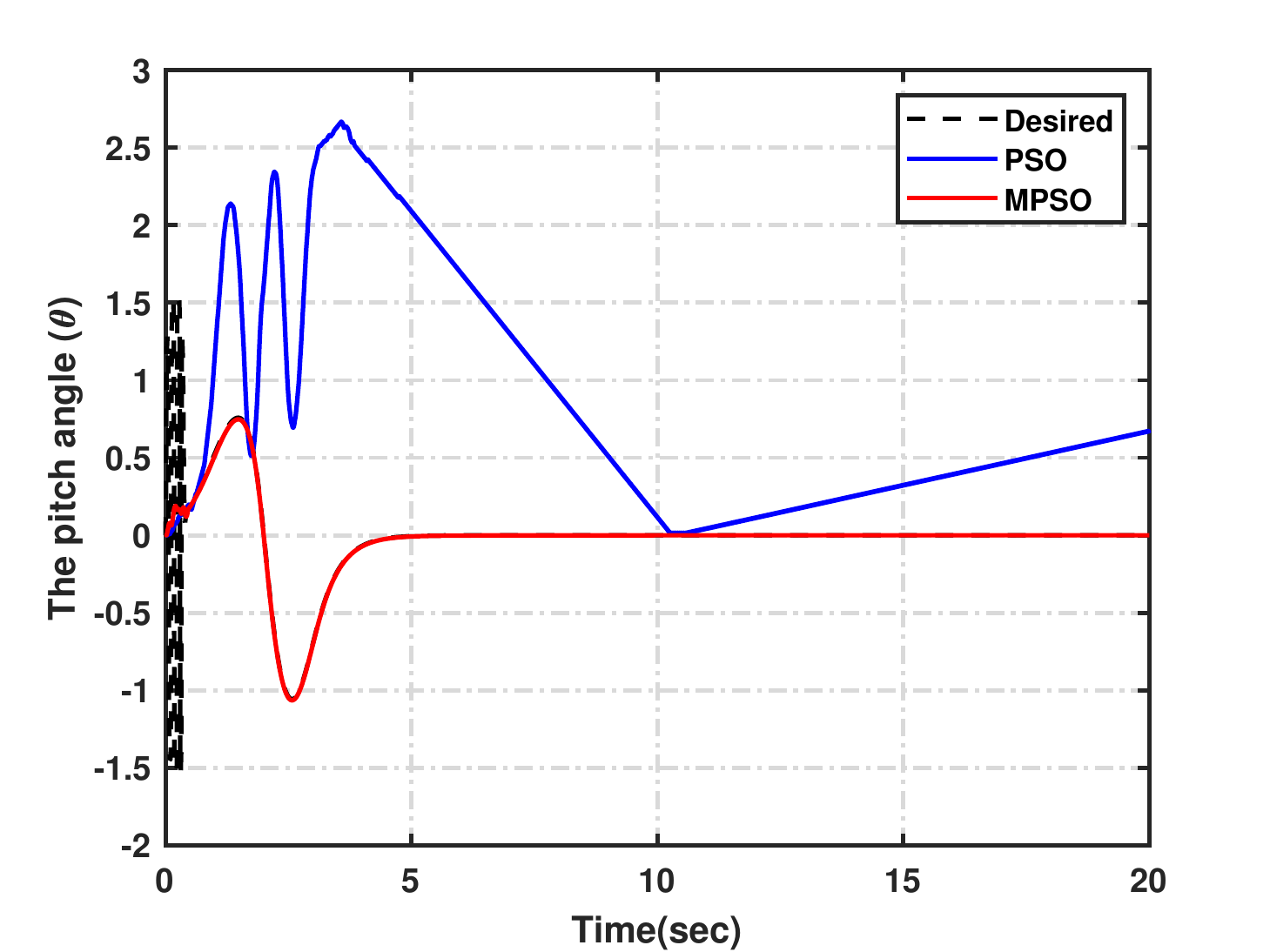}
			\label{Fig5c}
		}
\caption{Simulation results under the helicopter's mass change (half mass): outputs of the system.}\label{Fig5}
\end{figure}
\begin{figure}[t!]%
\centering
		\subfigure[]{
			\includegraphics[width=7cm]{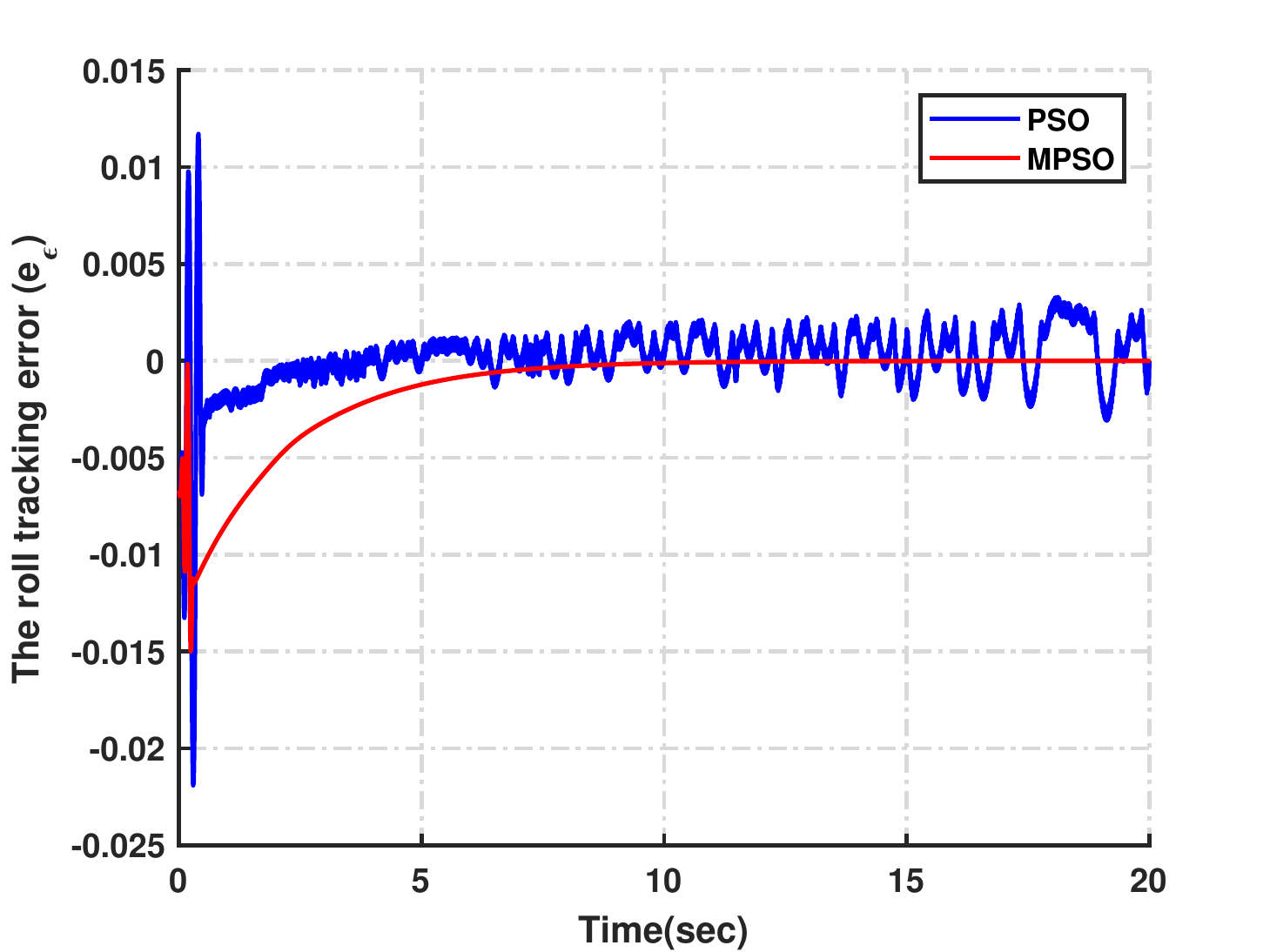}
			\label{Fig5d}	
		}
		\subfigure[]{
			\includegraphics[width=7cm]{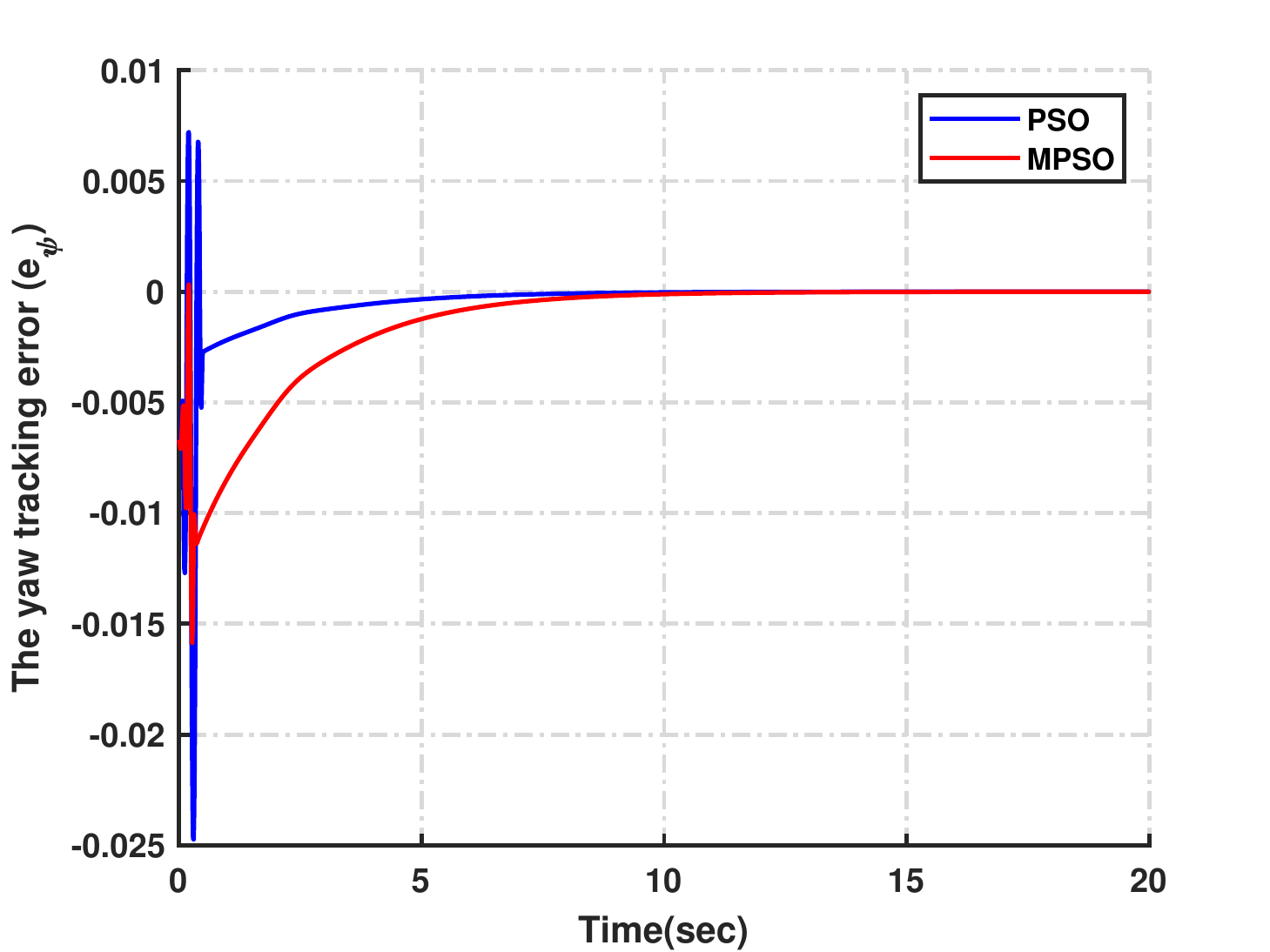}
			\label{Fig5e}
		}
		\subfigure[]{
			\includegraphics[width=7cm]{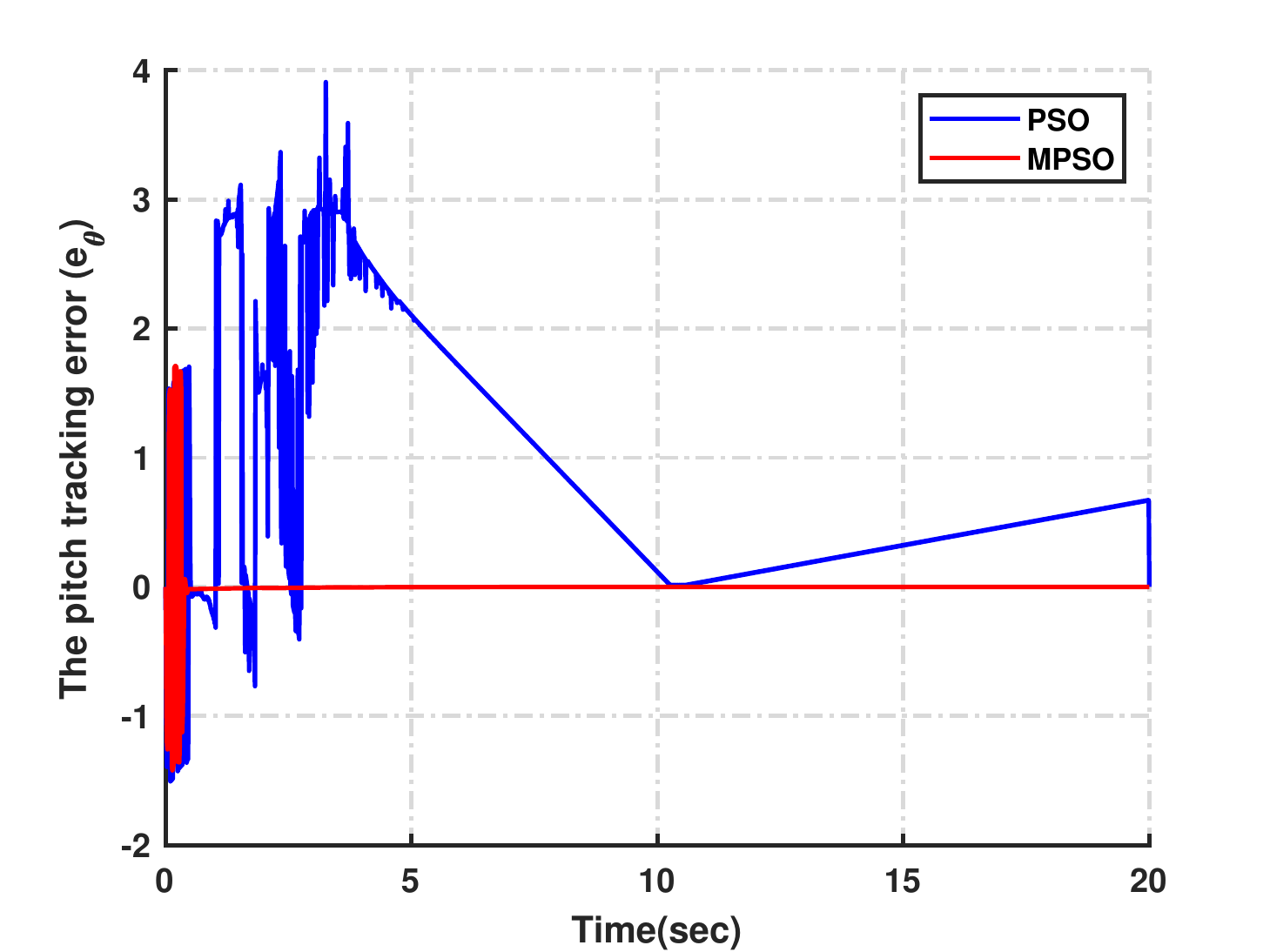}
			\label{Fig5f}
		}
\caption{Simulation results under the helicopter's mass change (half mass): motion tracking errors.}\label{Fig5.2}
\end{figure}
\begin{figure}[t!]%
\centering
		\subfigure[]{
			\includegraphics[width=7cm]{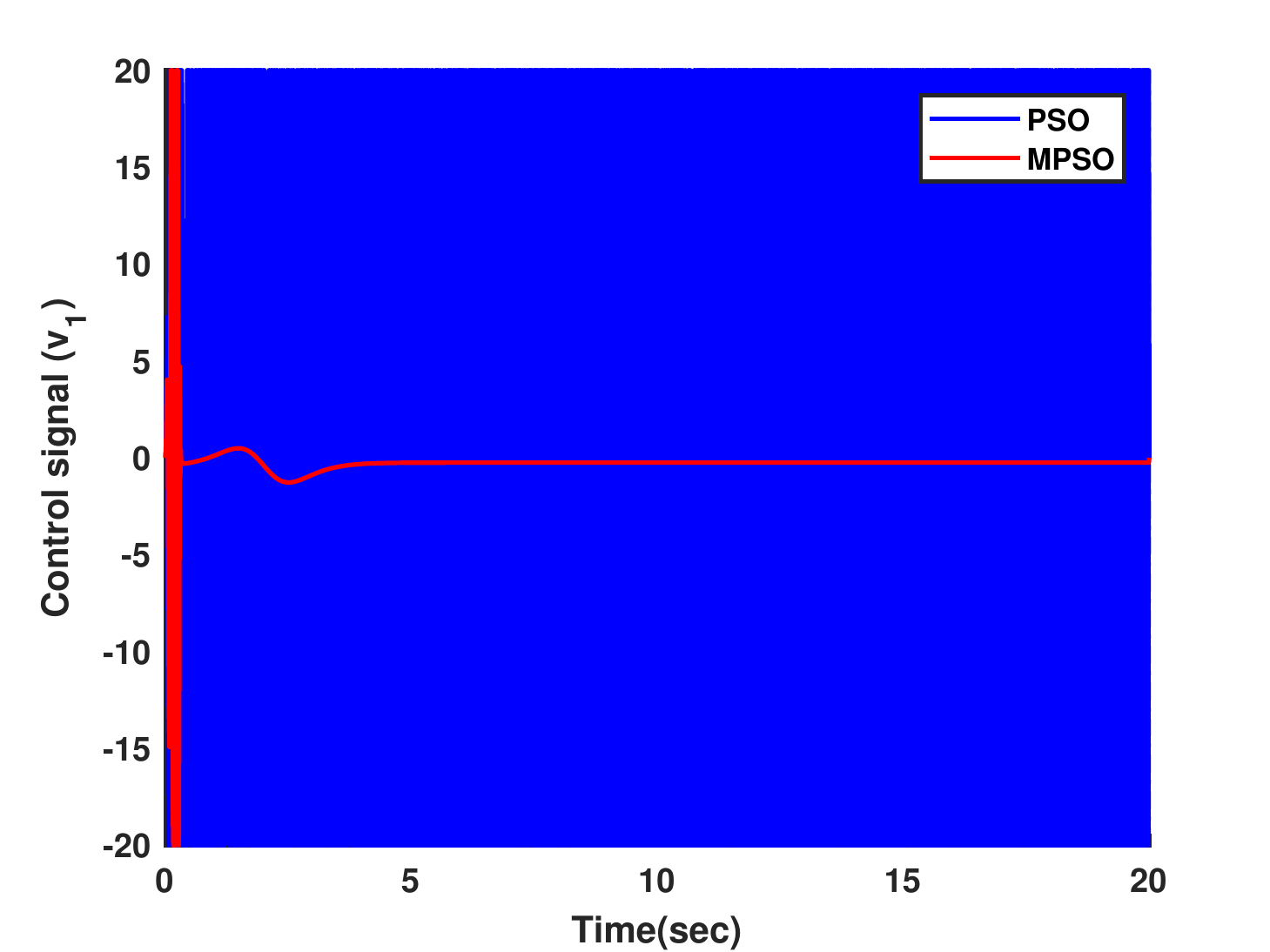}
			\label{Fig5g}	
		}
		\subfigure[]{
			\includegraphics[width=7cm]{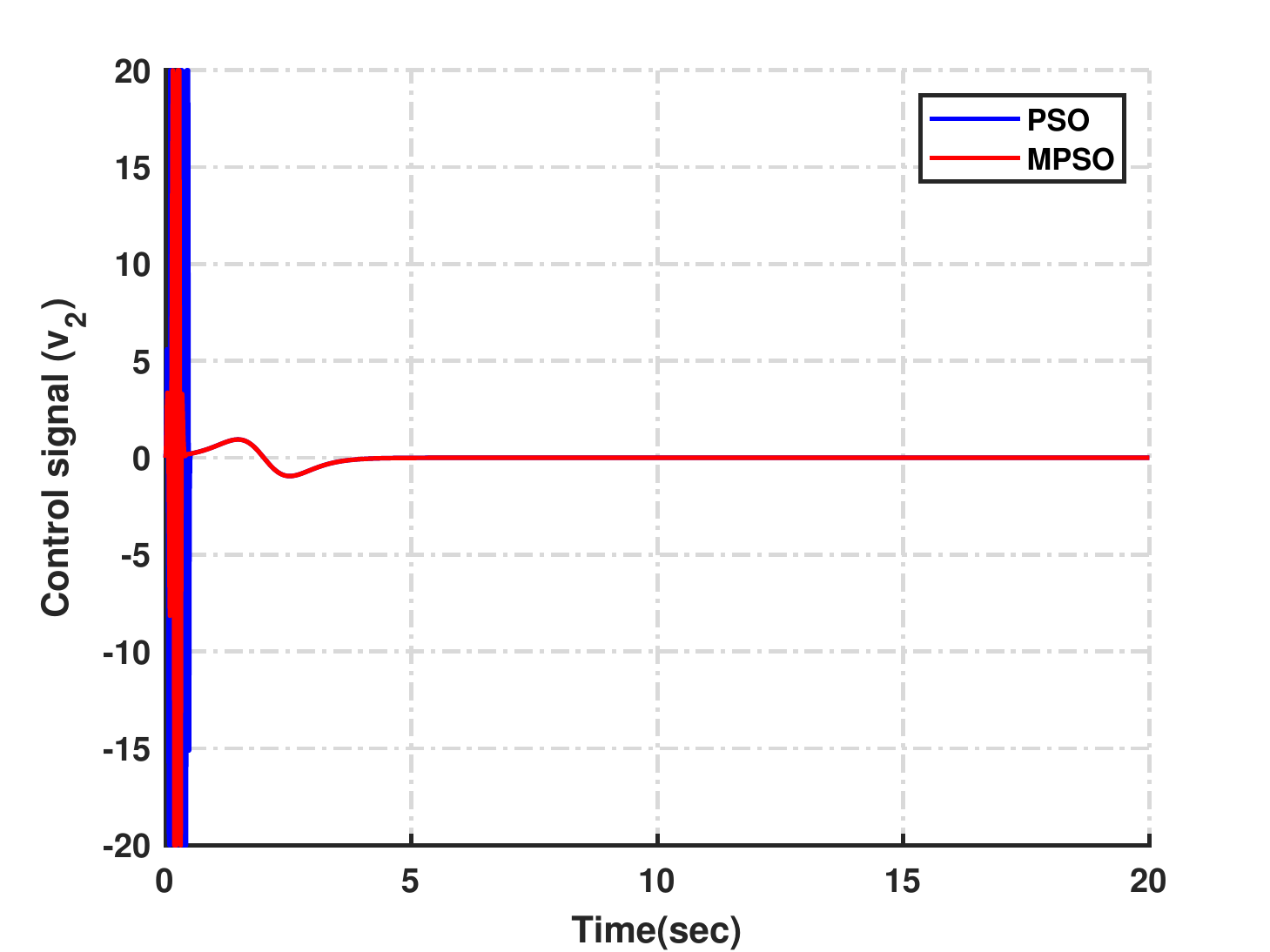}
			\label{Fig5h}
		}
		\subfigure[]{
			\includegraphics[width=7cm]{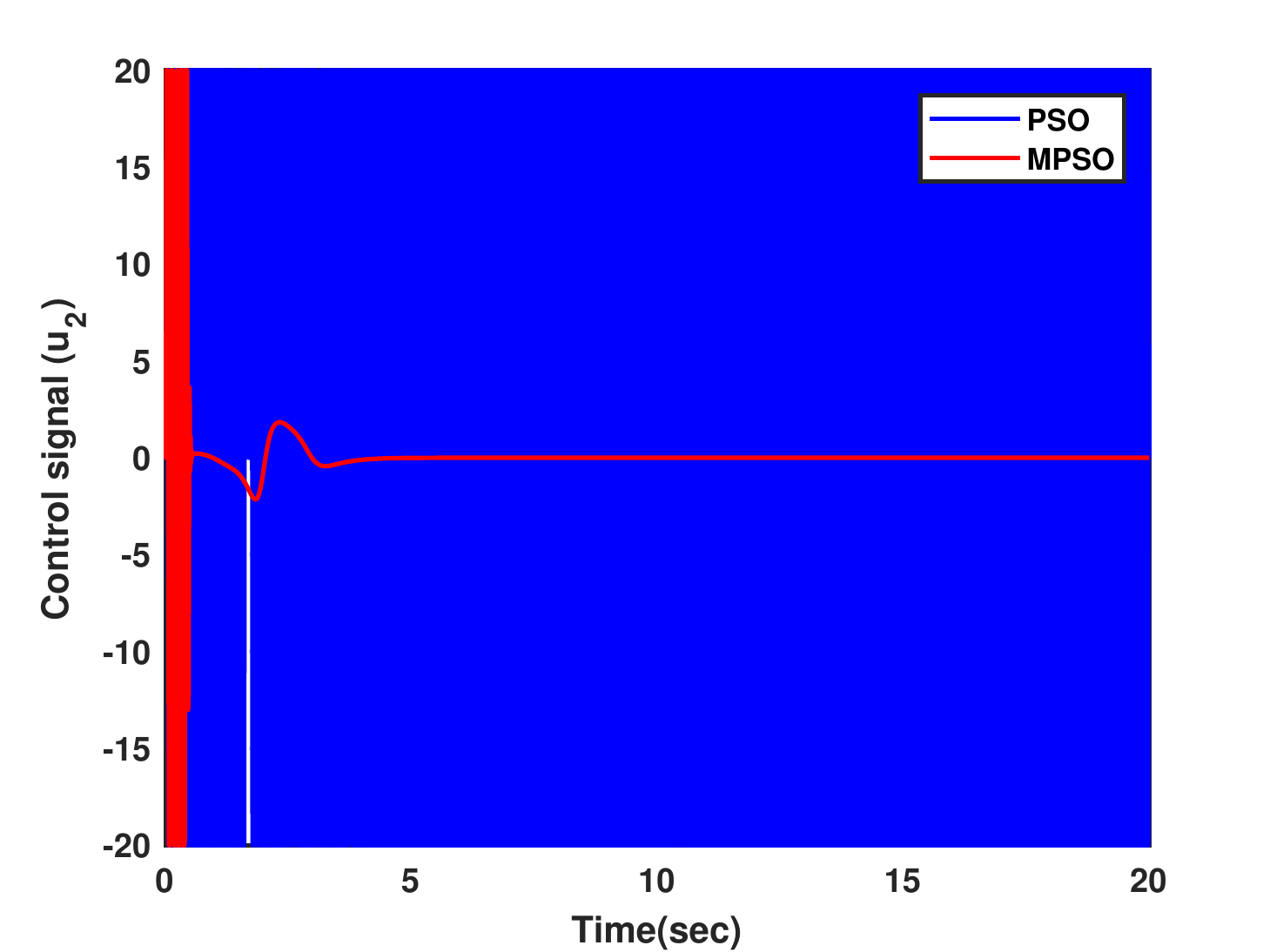}
			\label{Fig5i}
		}
\caption{Simulation results under the helicopter's mass change (half mass): control signals.}\label{Fig5.3}
\end{figure}

\begin{figure}[t!]%
\centering
		\subfigure[]{
			\includegraphics[width=7cm]{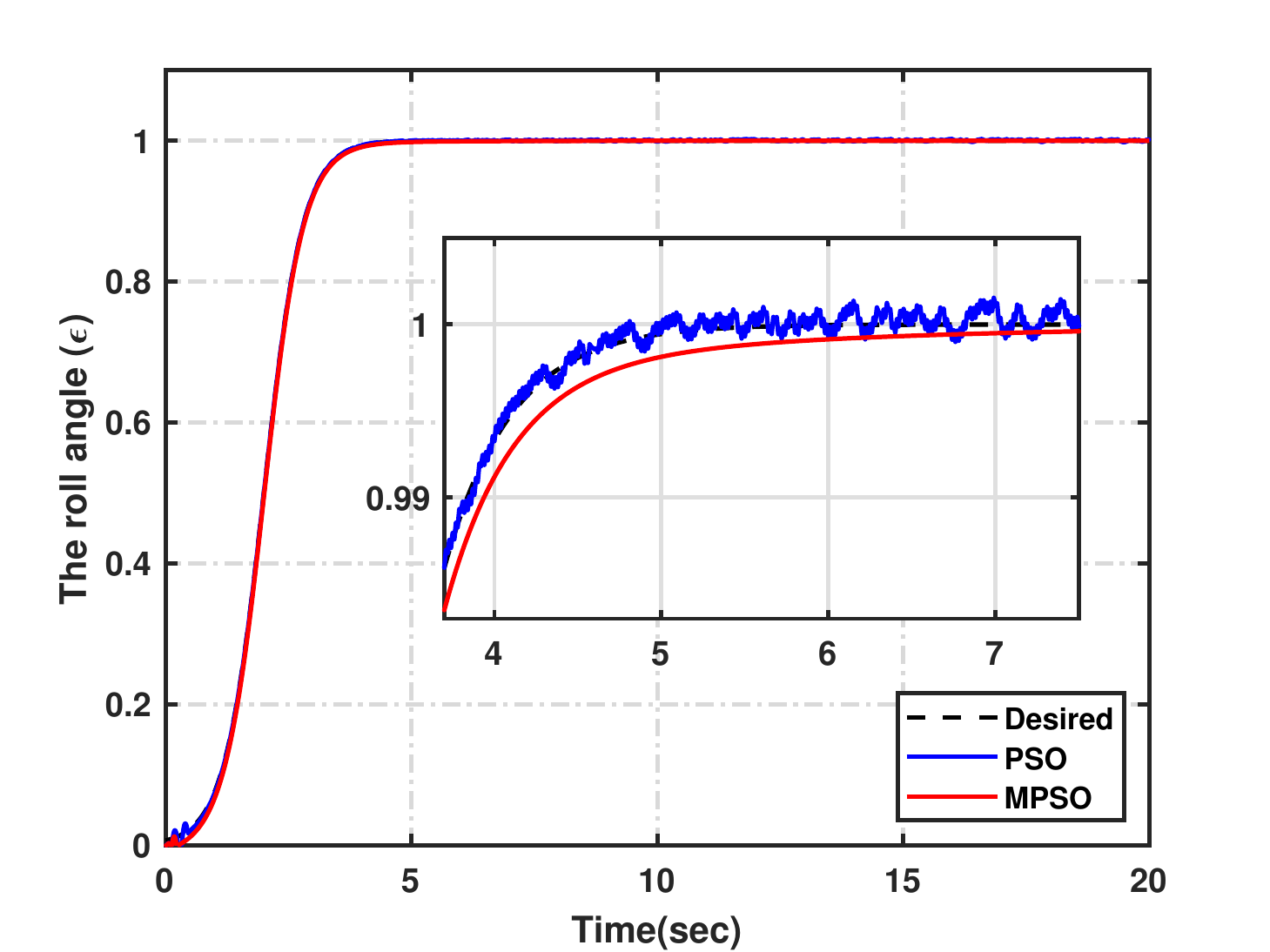}
			\label{Fig6a}	
		}
		\subfigure[]{
			\includegraphics[width=7cm]{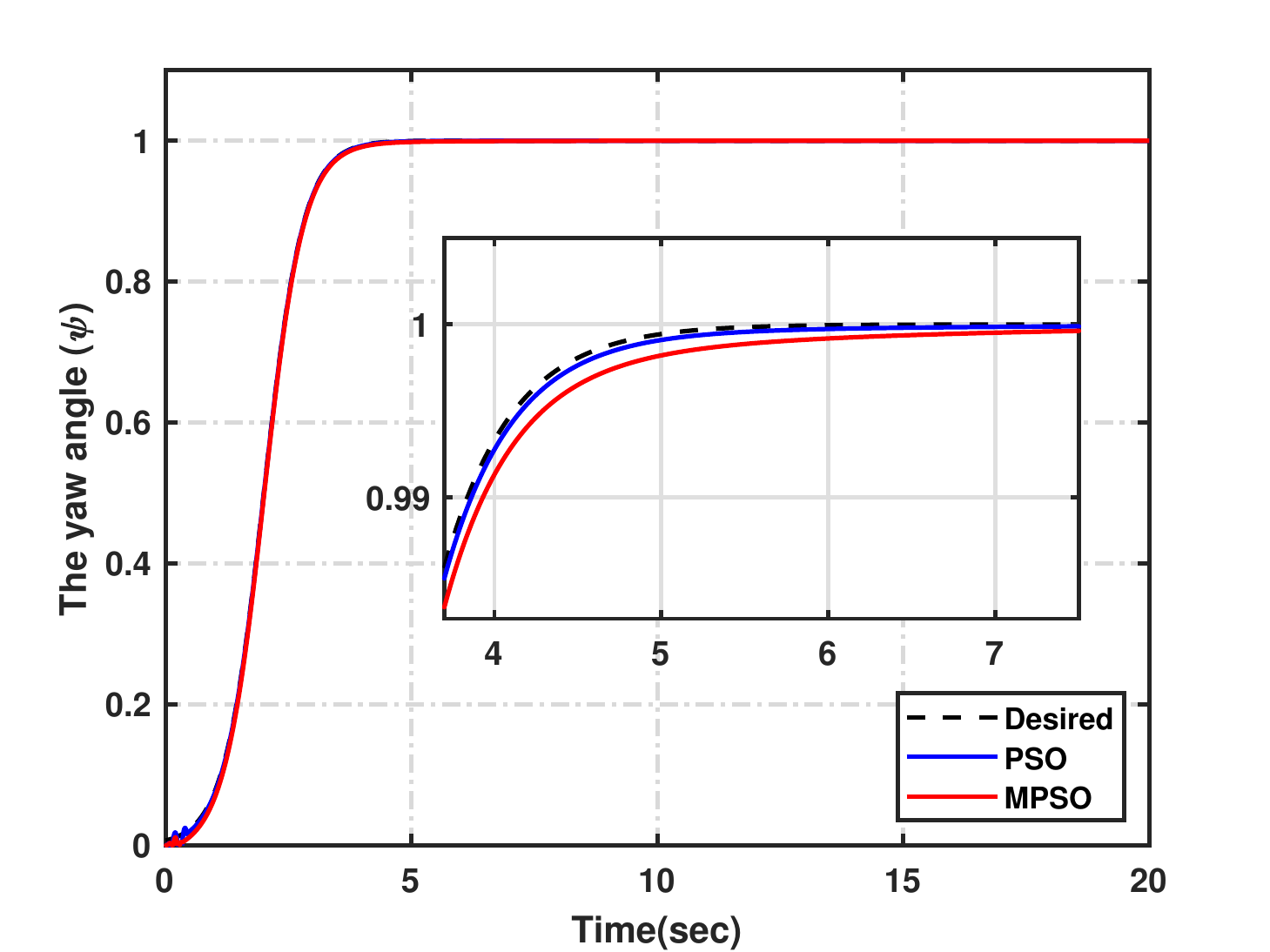}
			\label{Fig6b}
		}
		\subfigure[]{
			\includegraphics[width=7cm]{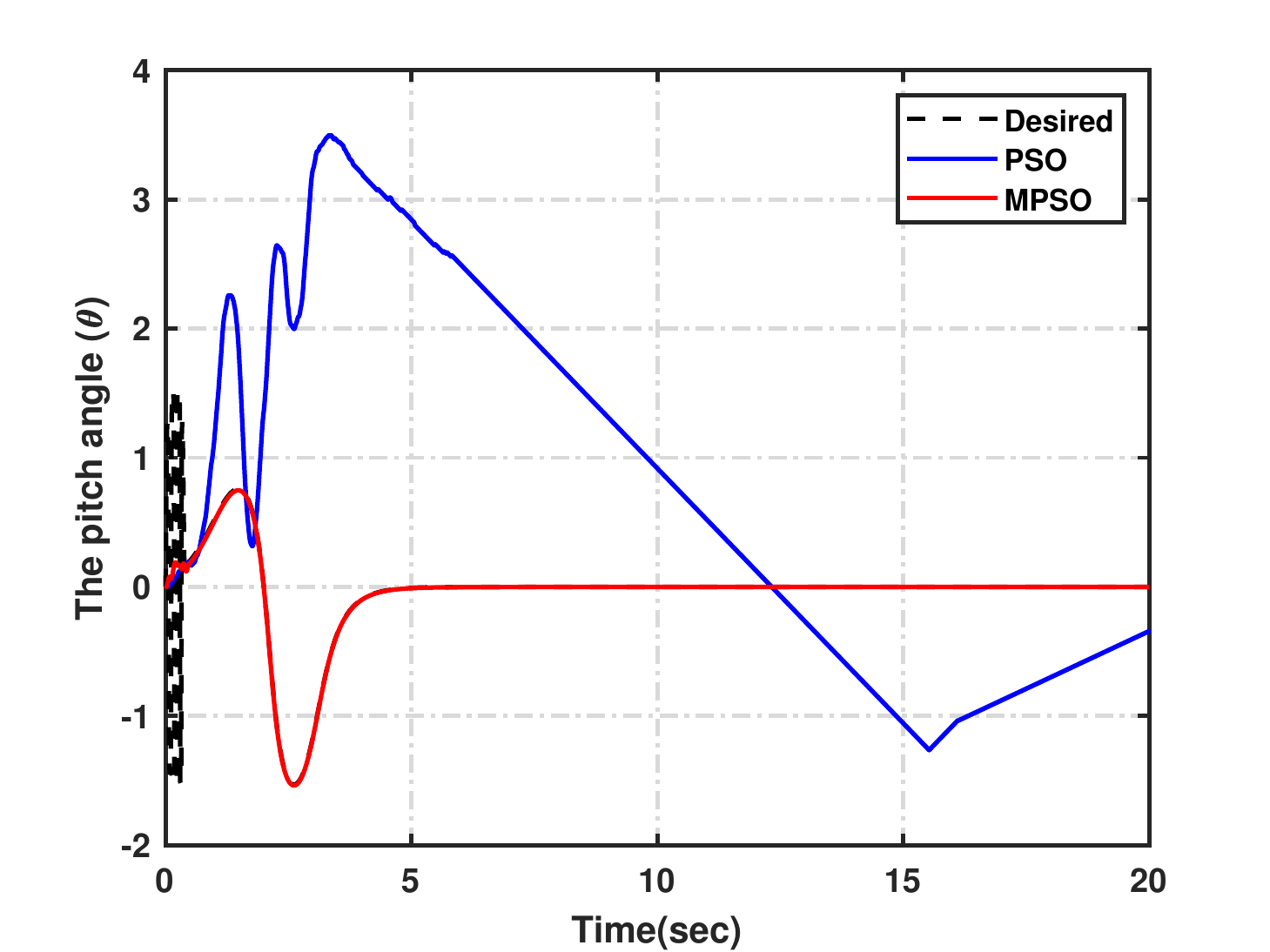}
			\label{Fig6c}
		}
\caption{Simulation results under the helicopter's mass change (one and a half times mass): outputs of the system.}\label{Fig6}
\end{figure}
\begin{figure}[t!]%
\centering
		\subfigure[]{
			\includegraphics[width=7cm]{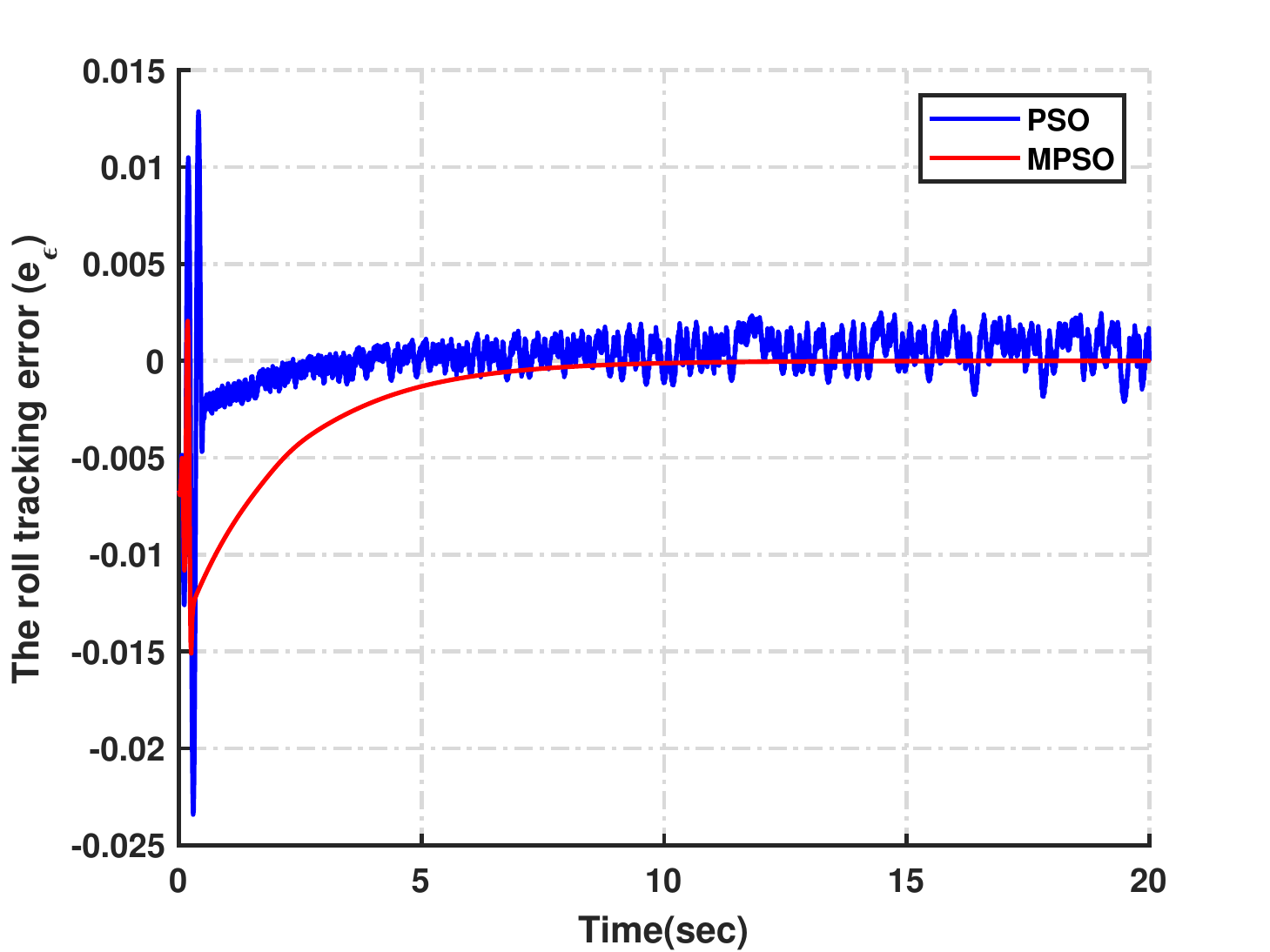}
			\label{Fig6d}	
		}
		\subfigure[]{
			\includegraphics[width=7cm]{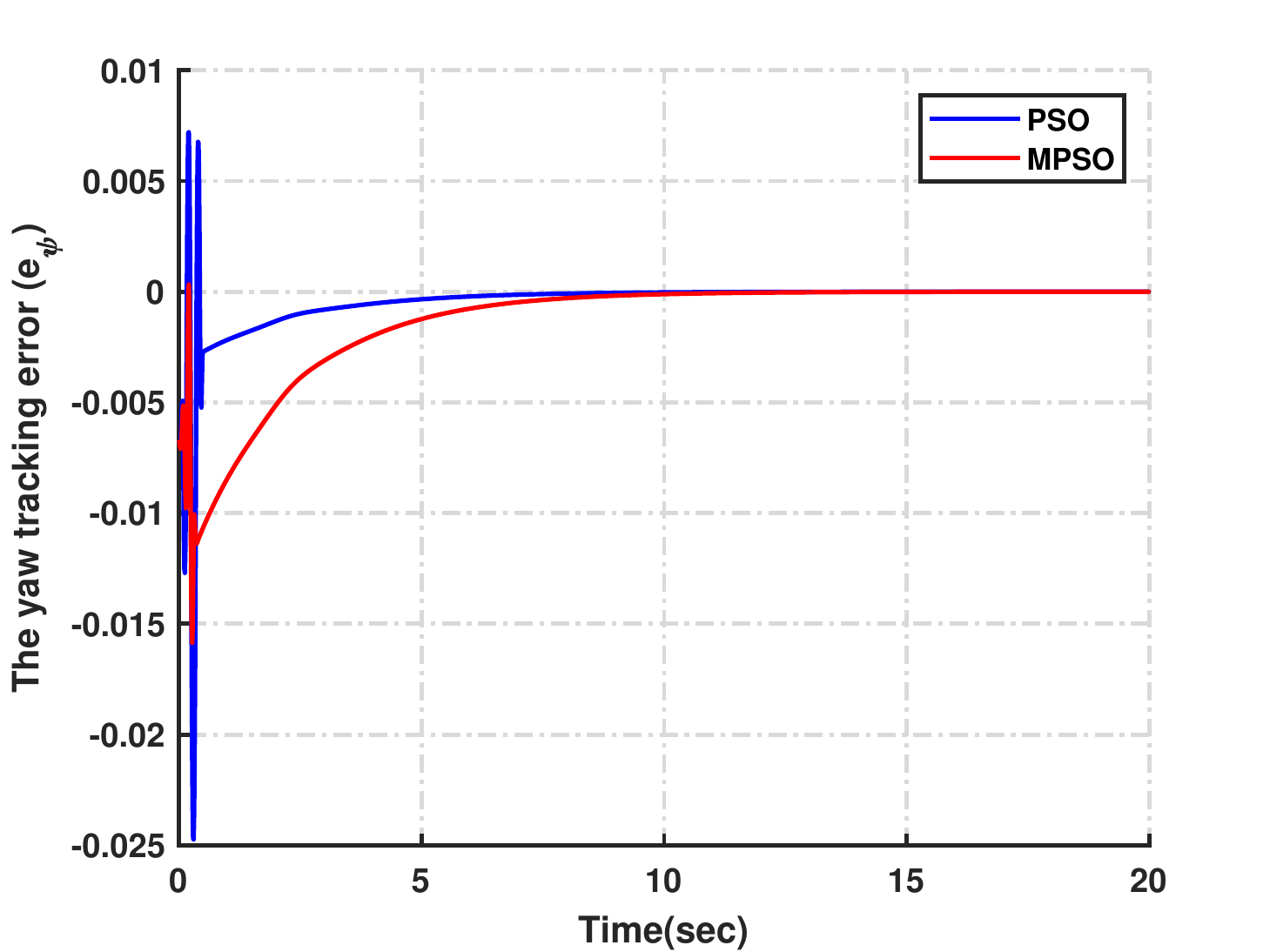}
			\label{Fig6e}
		}
		\subfigure[]{
			\includegraphics[width=7cm]{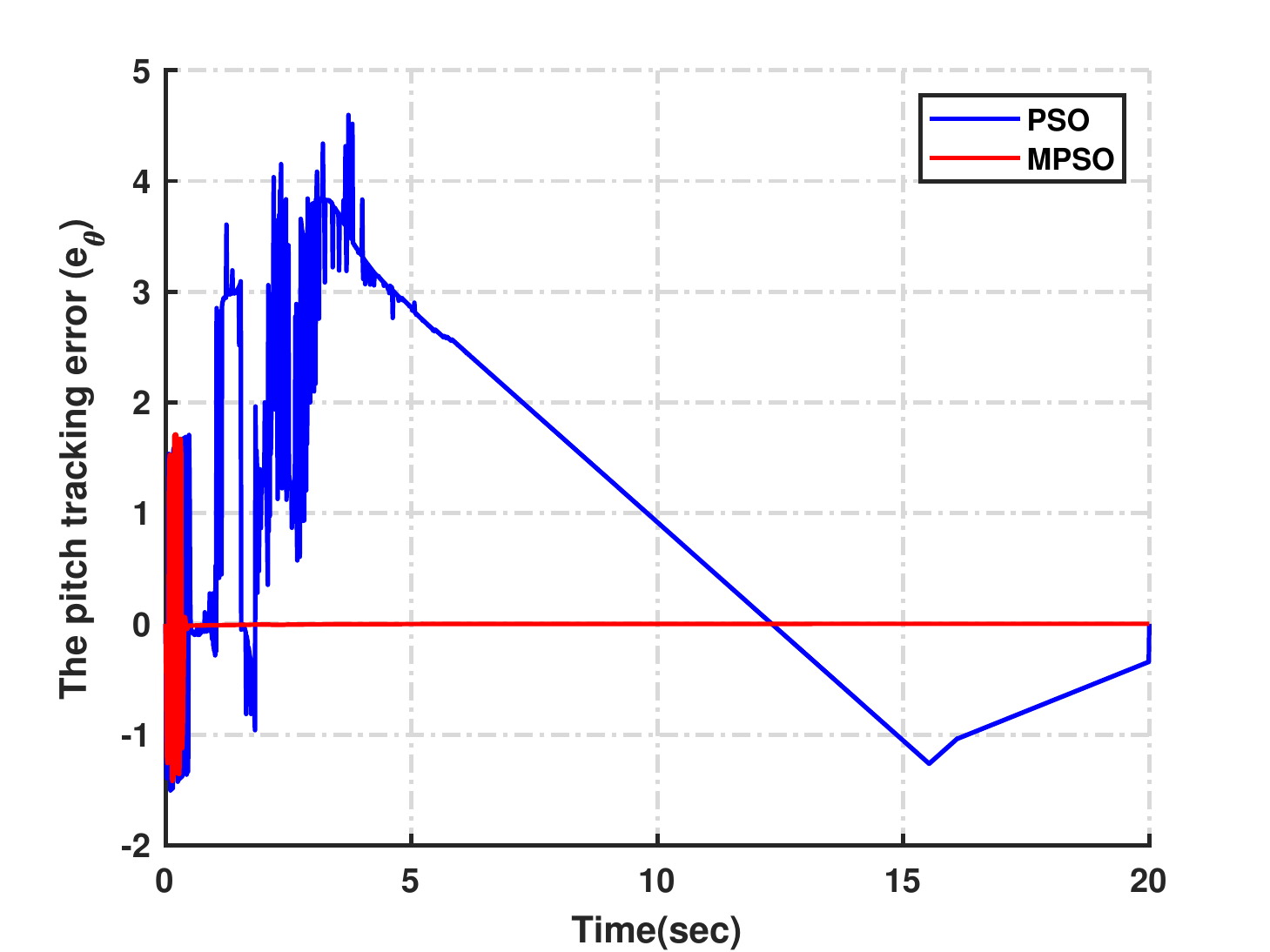}
			\label{Fig6f}
		}
\caption{Simulation results under the helicopter's mass change (one and a half times mass): motion tracking errors.}\label{Fig6.2}
\end{figure}
\begin{figure}[t!]%
\centering
		\subfigure[]{
			\includegraphics[width=7cm]{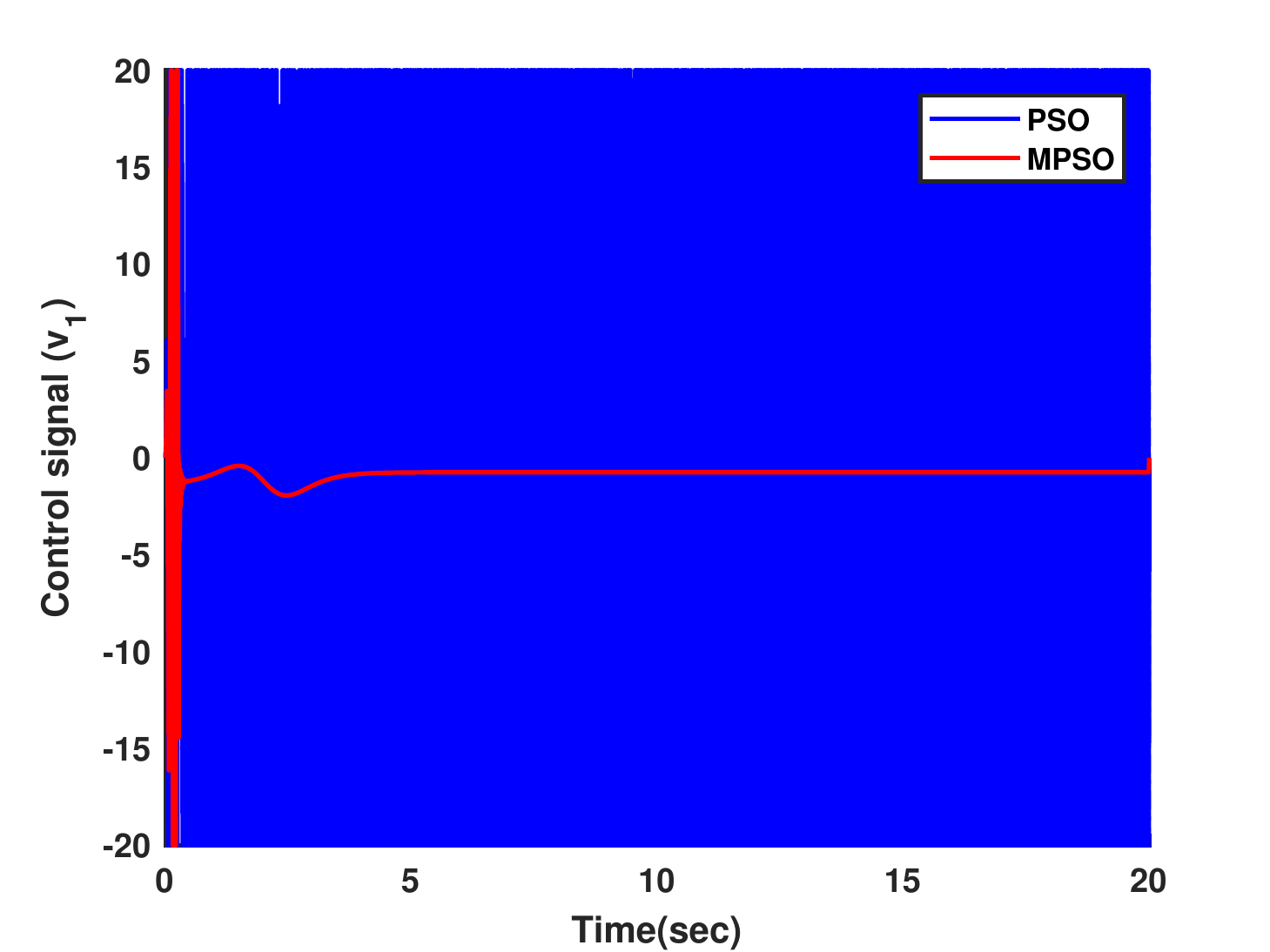}
			\label{Fig6g}	
		}
		\subfigure[]{
			\includegraphics[width=7cm]{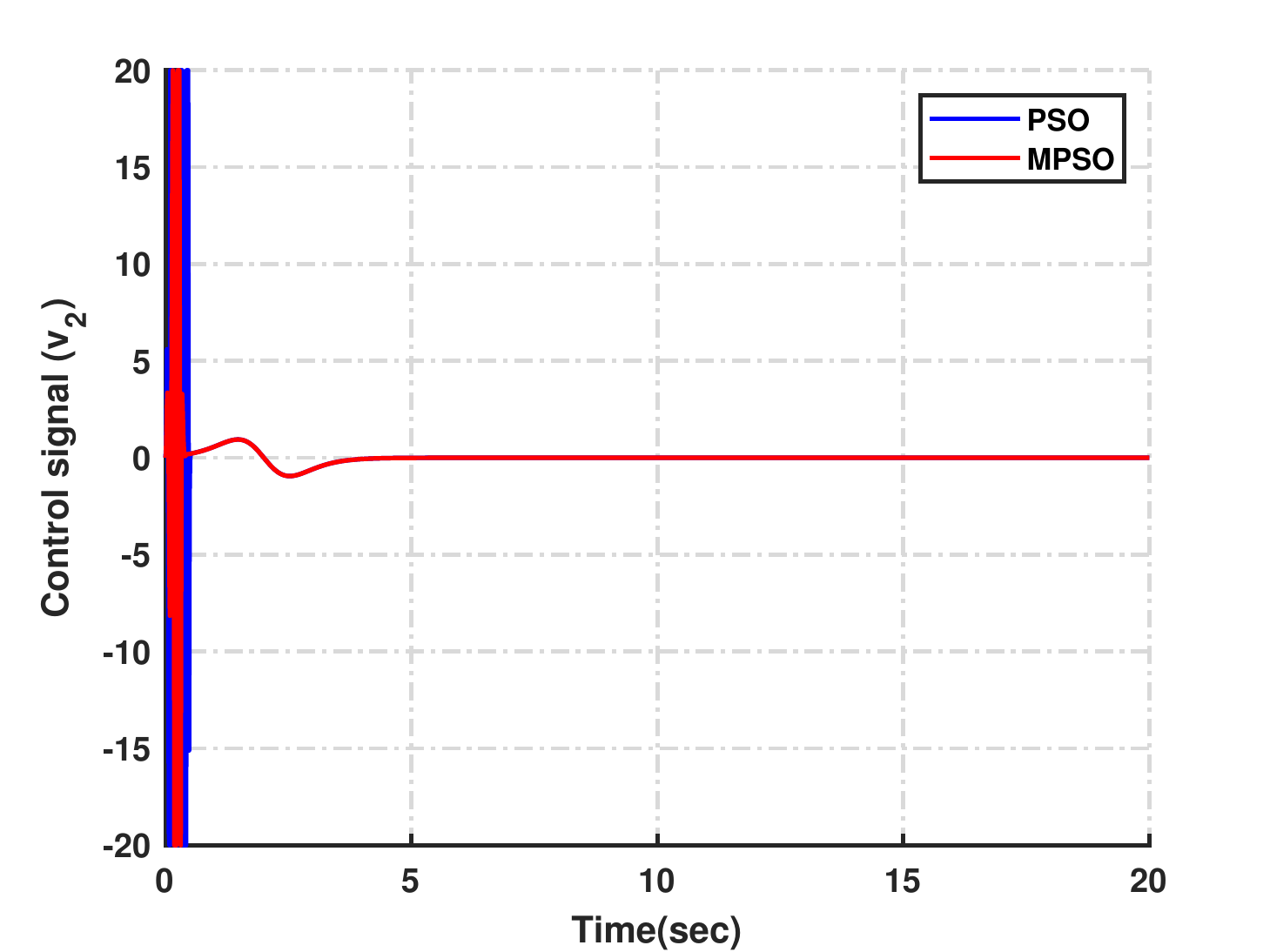}
			\label{Fig6h}
		}
		\subfigure[]{
			\includegraphics[width=7cm]{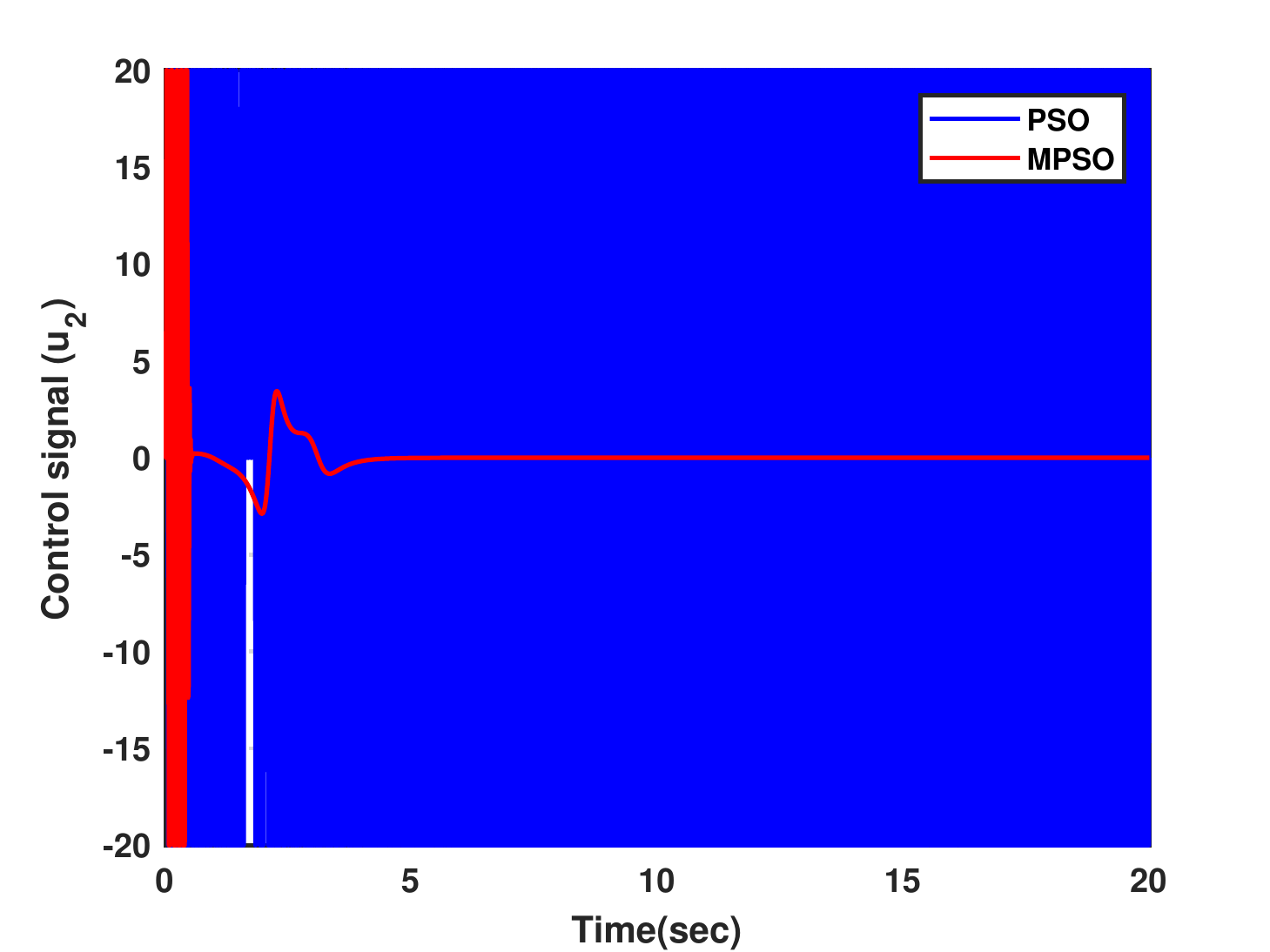}
			\label{Fig6i}
		}
\caption{Simulation results under the helicopter's mass change (one and a half times mass): control signals.}\label{Fig6.3}
\end{figure}

Moreover, we performed an additional test; we applied step disturbances of 1, 1, and 0.1 to the roll, pitch, and yaw angles at 12, 14, and 16 seconds, respectively. \hyperref[Fig7]{Fig. \ref{Fig7}} to \hyperref[Fig7]{Fig. \ref{Fig7.3}} show the results. Although applying the disturbance to the yaw angle leaves a small pitch tracking error, the controller is able to decay the roll and pitch tracking errors to zero without fluctuation or unstable behavior.  

\begin{figure}[t!]%
\centering
		\subfigure[]{
			\includegraphics[width=7cm]{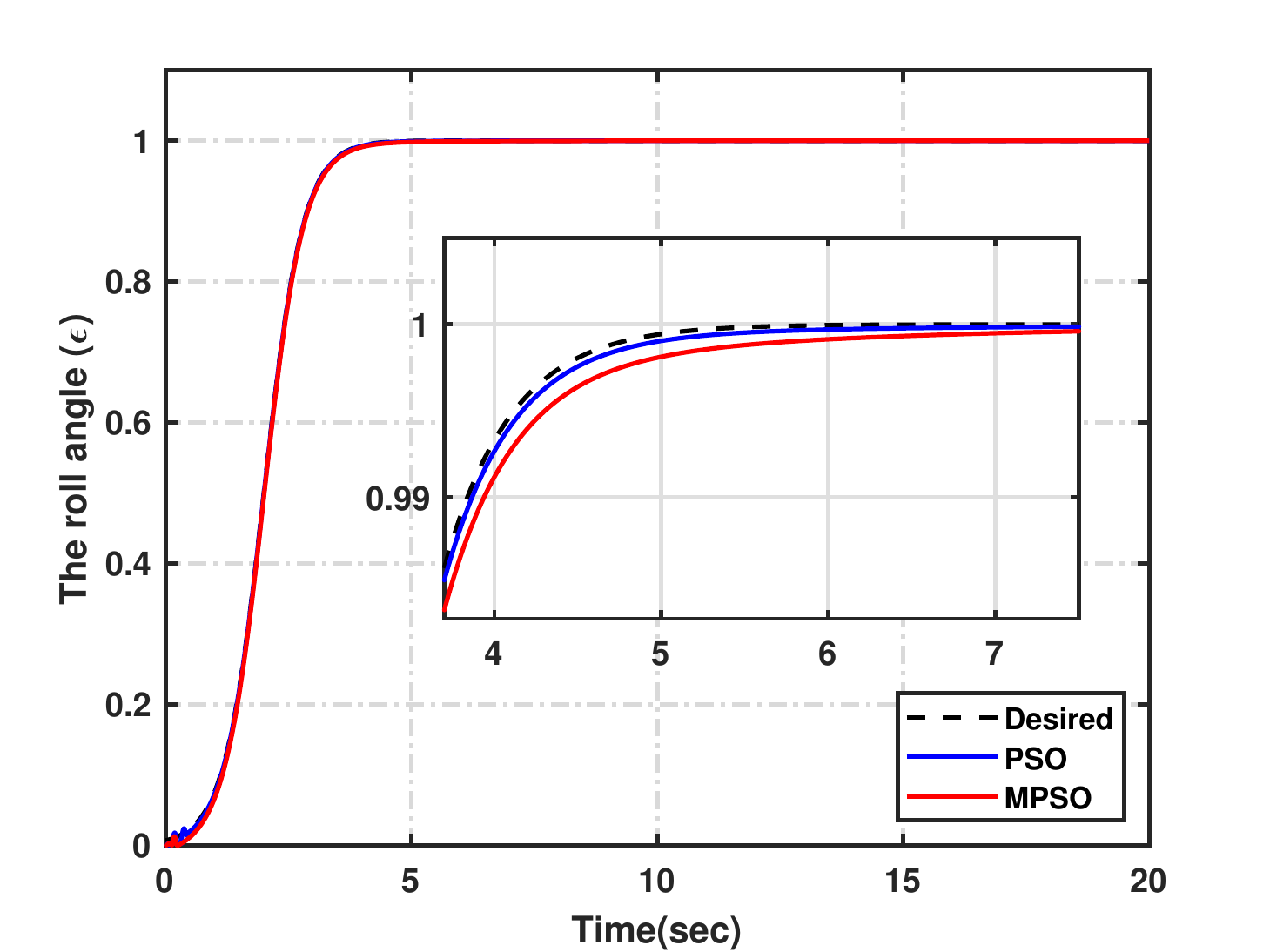}
			\label{Fig7a}	
		}
		\subfigure[]{
			\includegraphics[width=7cm]{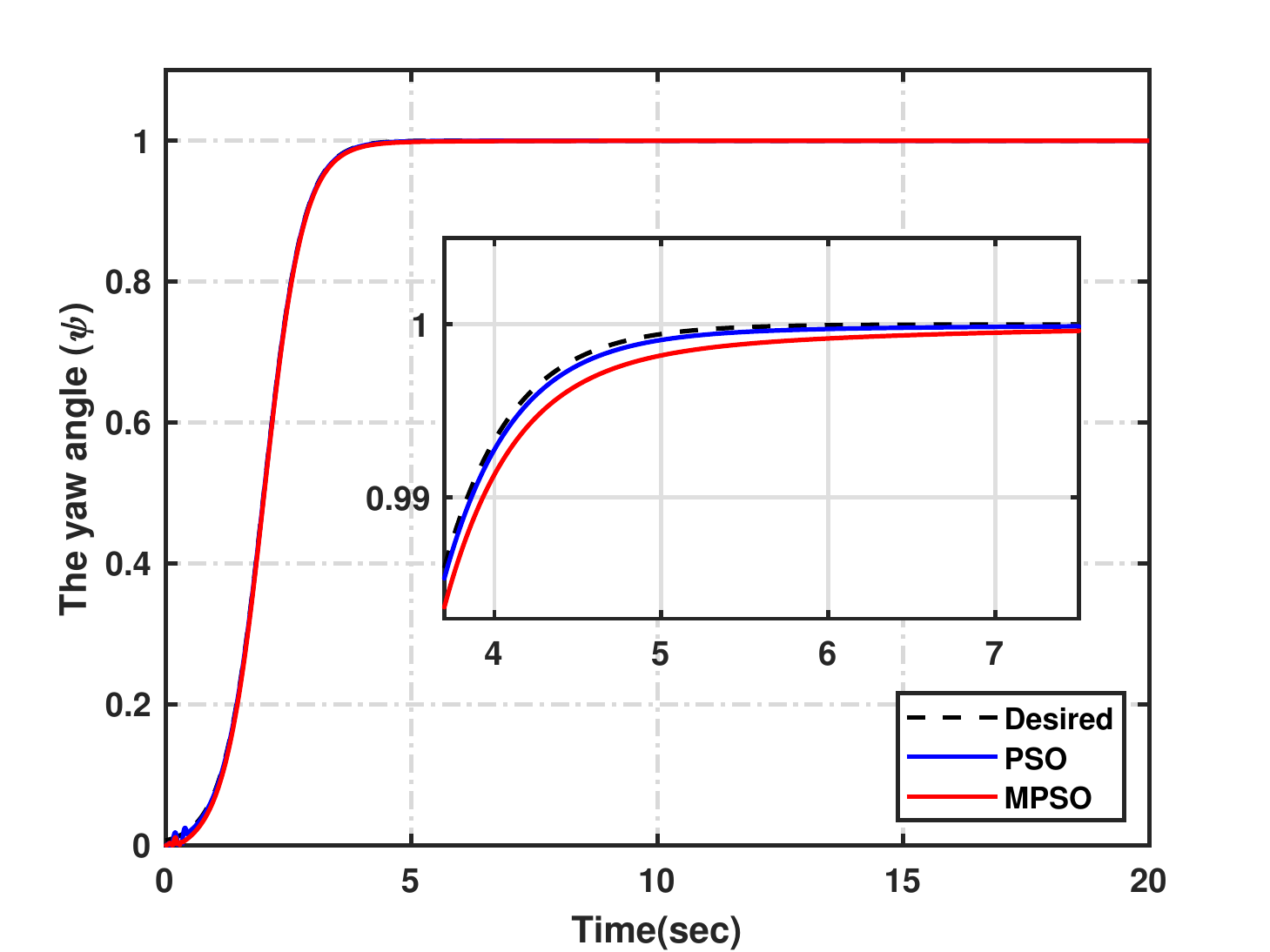}
			\label{Fig7b}
		}
		\subfigure[]{
			\includegraphics[width=7cm]{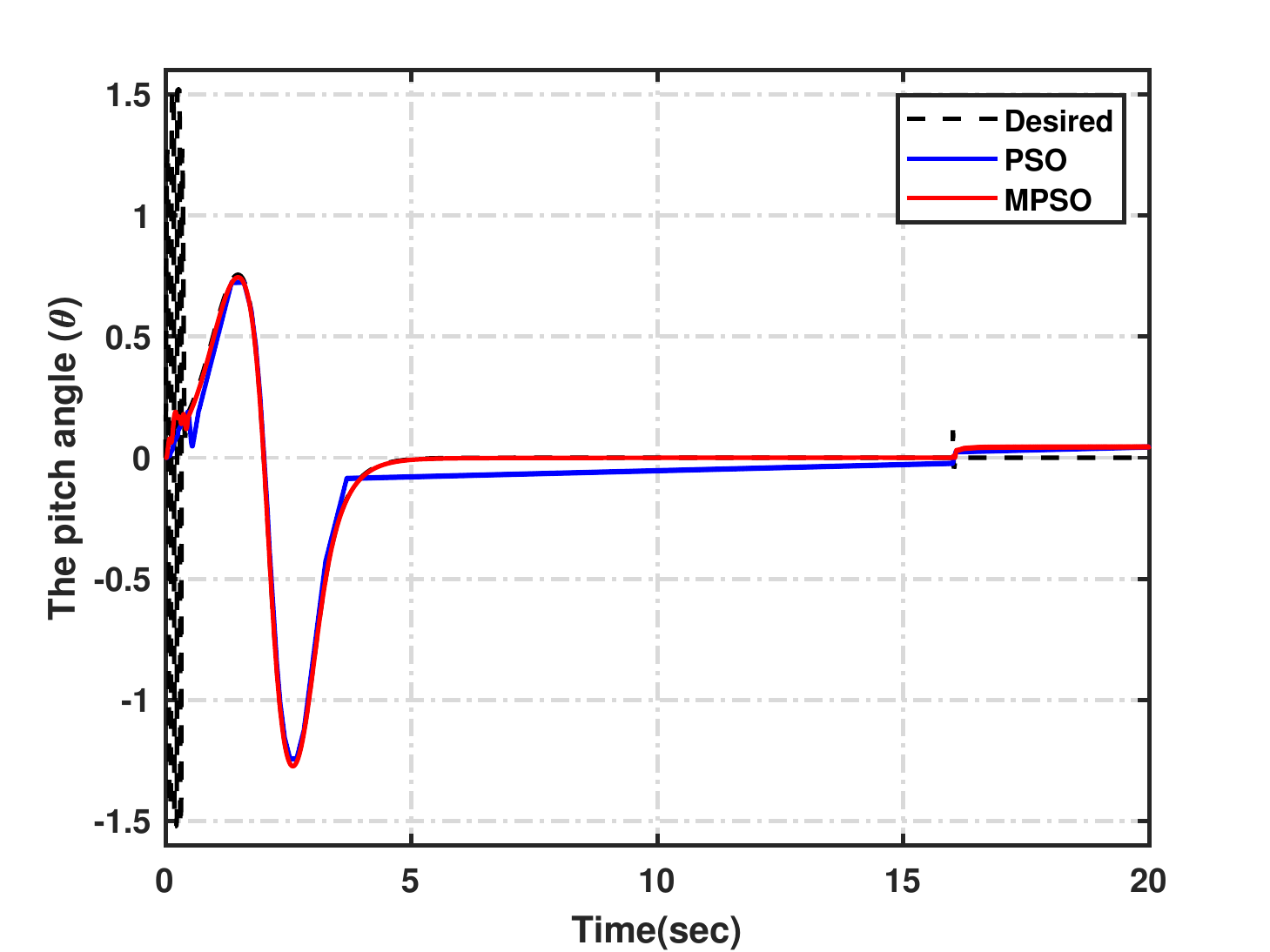}
			\label{Fig7c}
		}
\caption{Simulation results under disturbances: outputs of the system.}\label{Fig7}
\end{figure}
\begin{figure}[t!]%
\centering
		\subfigure[]{
			\includegraphics[width=7cm]{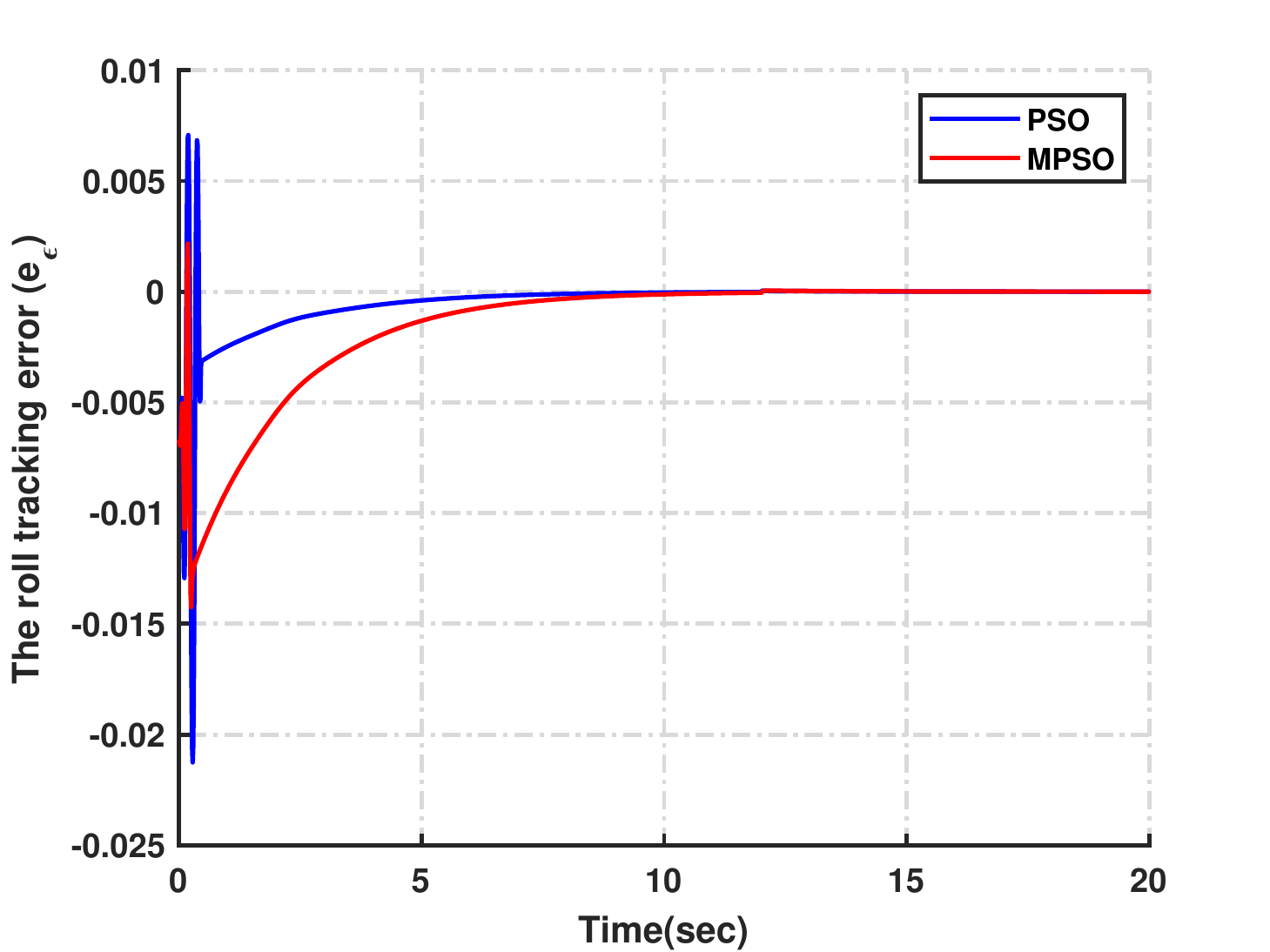}
			\label{Fig7d}	
		}
		\subfigure[]{
			\includegraphics[width=7cm]{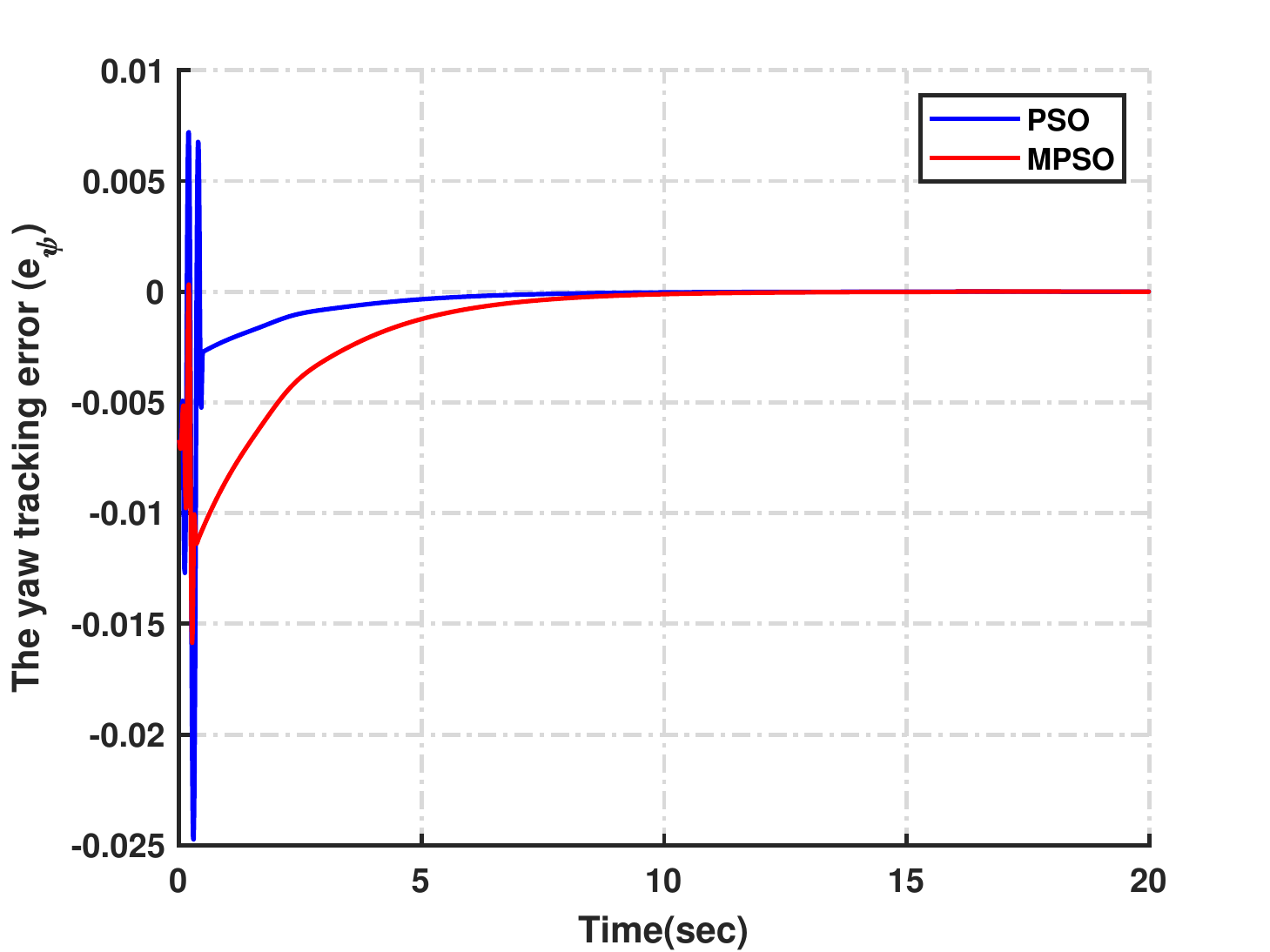}
			\label{Fig7e}
		}
		\subfigure[]{
			\includegraphics[width=7cm]{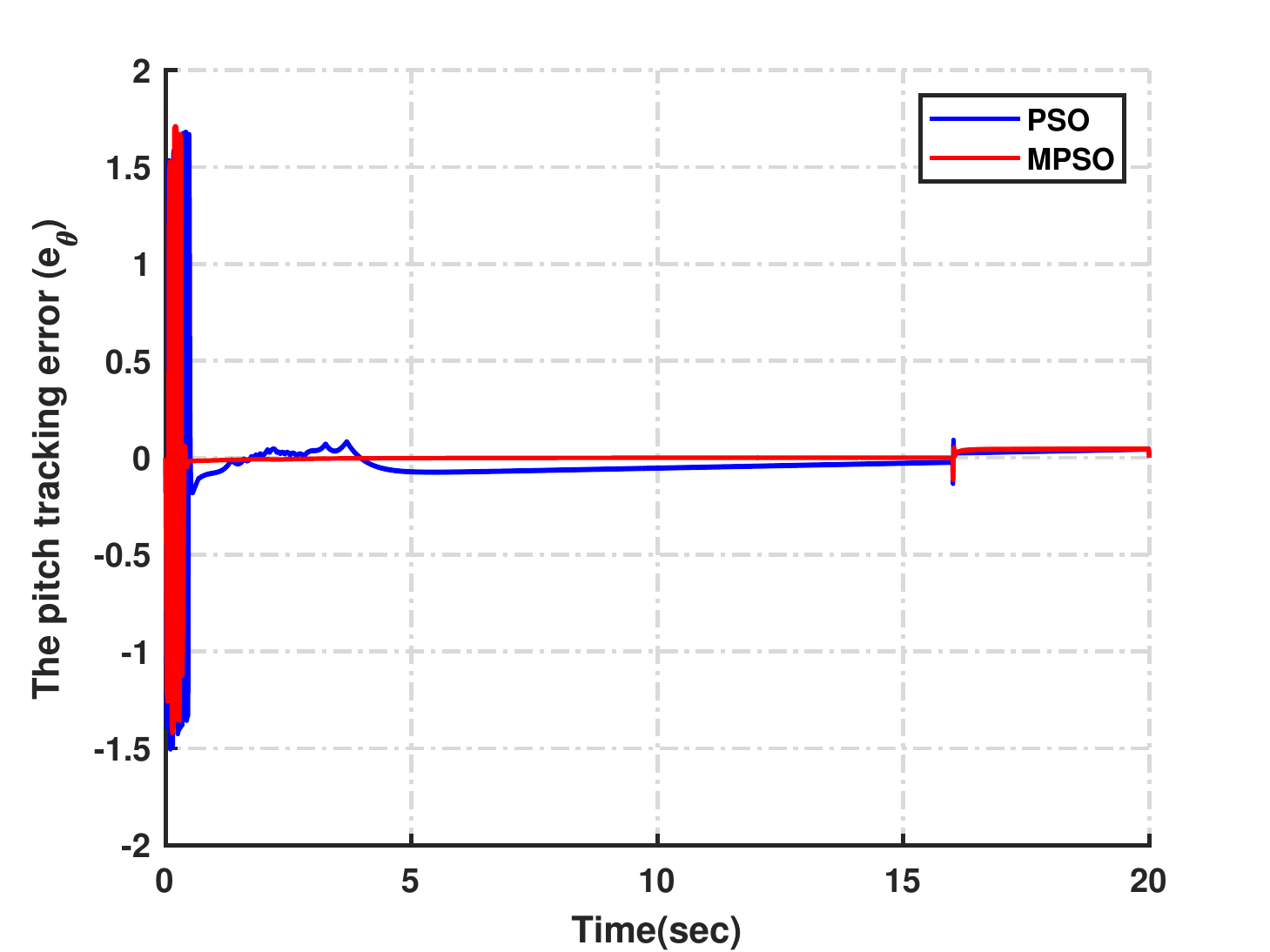}
			\label{Fig7f}
		}
\caption{Simulation results under disturbances: motion tracking errors.}\label{Fig7.2}
\end{figure}
\begin{figure}[t!]%
\centering
		\subfigure[]{
			\includegraphics[width=7cm]{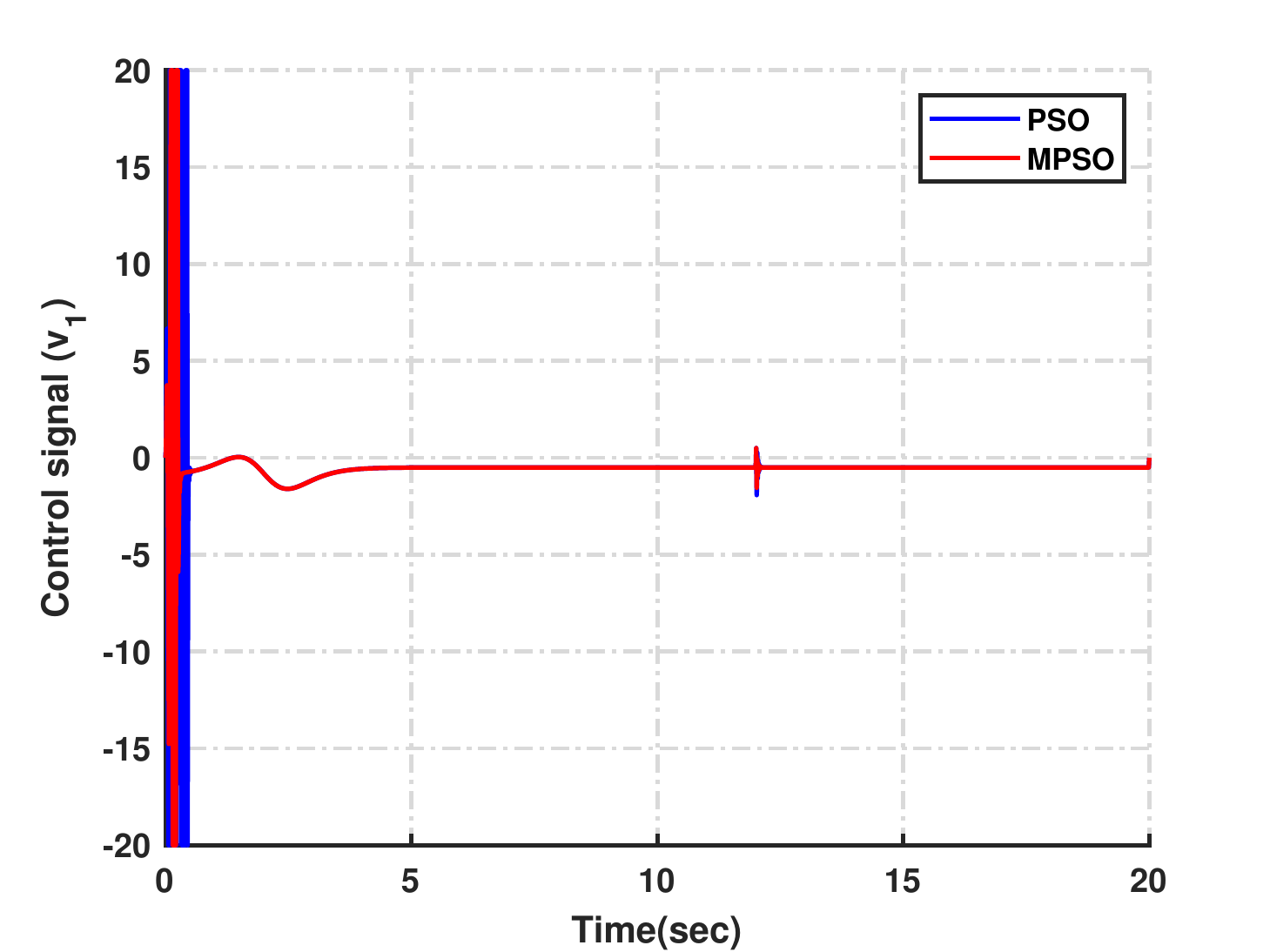}
			\label{Fig7g}	
		}
		\subfigure[]{
			\includegraphics[width=7cm]{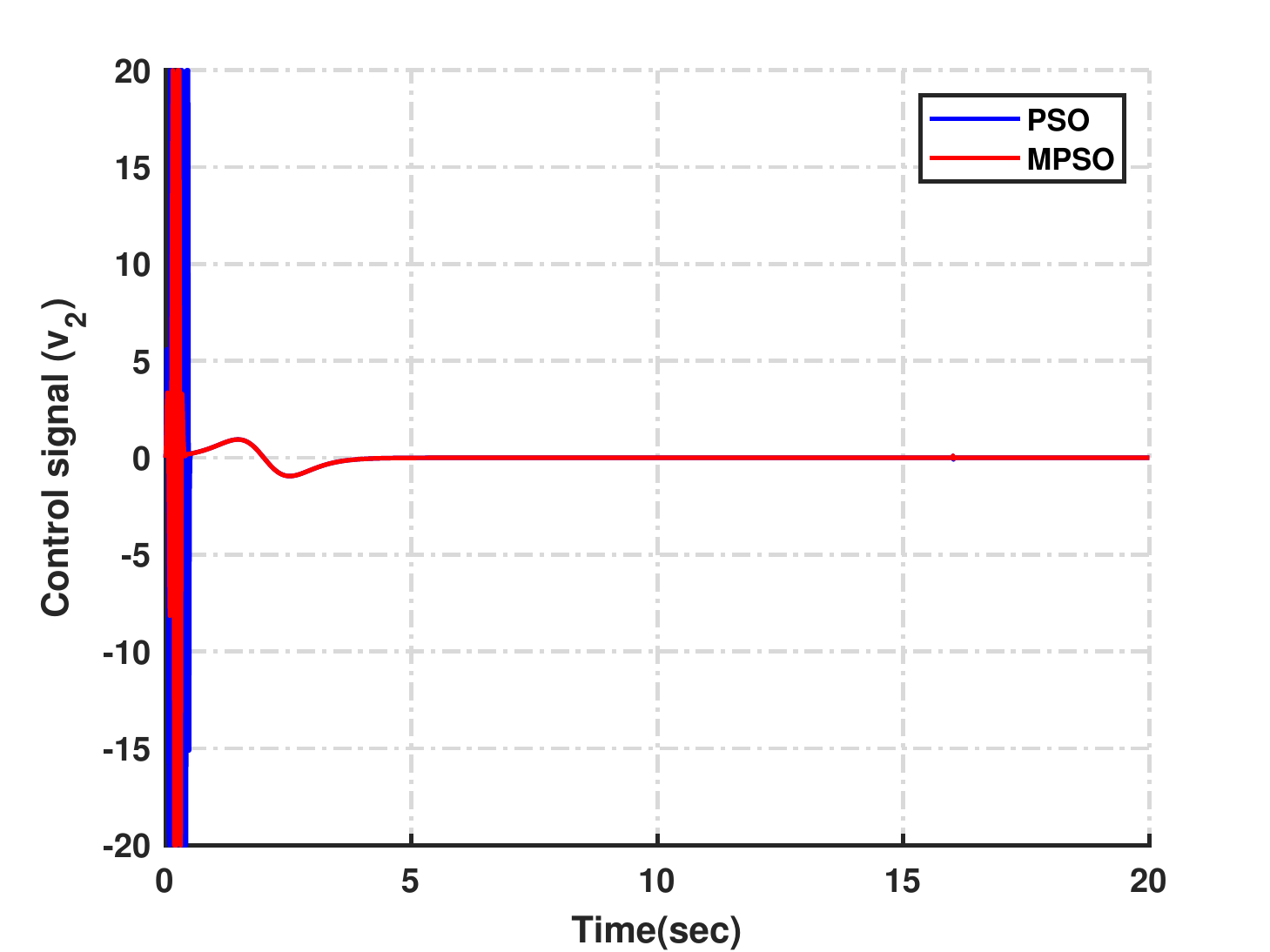}
			\label{Fig7h}
		}
		\subfigure[]{
			\includegraphics[width=7cm]{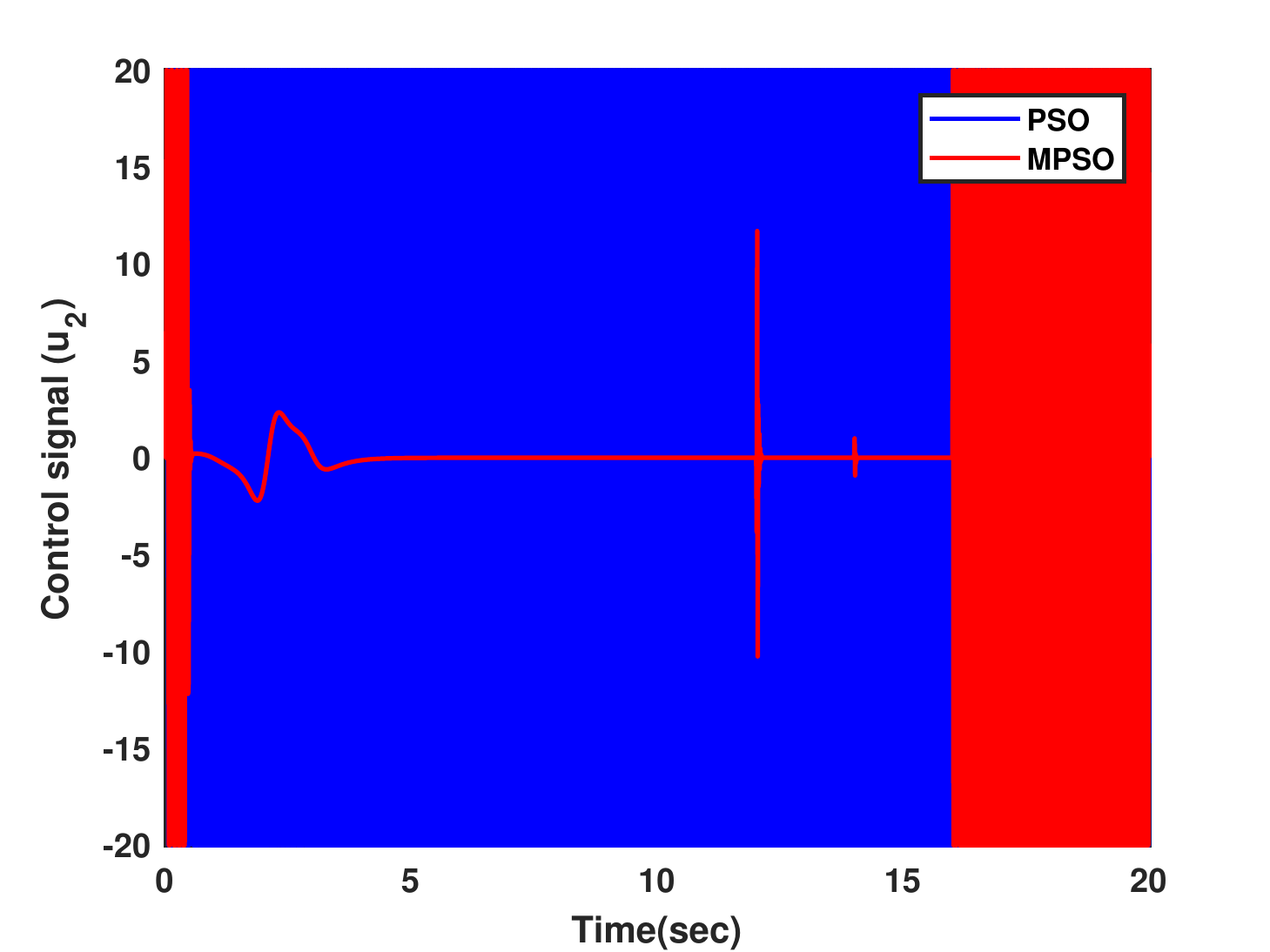}
			\label{Fig7i}
		}
\caption{Simulation results under disturbances: control signals.}\label{Fig7.3}
\end{figure}

\subsection{Further analysis of the MPSO algorithm performance against other metaheuristics}\label{subsec5.1}

To further investigate the effectiveness of the MPSO to optimize the 3-DOF helicopter control, we compare its results to the following algorithms' results: 

\begin{itemize}
\item[•] MPSO algorithm \cite{naderi2020designing}
\item[•] Standard PSO algorithm \cite{kennedy1995particle}
\item[•] Genetic algorithm (GA) \cite{michalewicz1996genetic}
\item[•] Ant colony optimization (ACO) \cite{dorigo2006ant}
\item[•] Imperialist competitive algorithm (ICA) \cite{atashpaz2007imperialist}
\item[•] Grey wolf optimizer (GWO) \cite{mirjalili2014grey}
\item[•] Bat algorithm (BA) \cite{yang2010new}
\item[•] Differential evolution (DE) \cite{storn1997differential}
\end{itemize}

\begin{table}[t!]
\begin{center}
%\begin{minipage}{174pt}
\caption{Upper and lower bound of search space for each parameter in MPSO, ACO, and ICA.}\label{table4}%
\begin{tabular}{@{}p{6cm} p{2.5cm} p{2cm}@{}}
\toprule
Parameters & Lower bound & Upper bound\\
\midrule
			$K_\epsilon, K_\psi, K_\theta$ & 0 & 200 \\
			$K_p, K_d$ & 0 & 100 \\
			Centers for all fuzzy membership functions & 0 & 10\\
			$\Gamma_\epsilon, \Gamma_\psi, \Gamma_\theta$ & 0 & 100 \\
\bottomrule
\end{tabular}
%\end{minipage}
\end{center}
\end{table}

All these algorithms are implemented to optimize the adaptive fuzzy logic controllers of the 3-DOF helicopter model. The population size is selected to be 30 for all algorithms. For those algorithms incapable of defining the upper and lower bound of search space for each parameter, [0 200] is chosen as the search space. For other algorithms (MPSO, ACO, and ICA), \hyperref[table4]{Table \ref{table4}} shows the upper and lower bound of search space for each parameter. The cost function is also the same for all RMSE. Each algorithm has been run 25 times in MATLAB. The best values of the fitness function (the lowest ones) and the mean of fitness function values for the 25 runs with 500 iterations are reported in \hyperref[table5]{Table \ref{table5}}. The elapsed time in the table refers to the average time required to complete one run of each algorithm. According to the results in the table, the MPSO algorithm has the lowest cost function value among others. As the cost function is RMSE, the MPSO provides controllers' gains and parameters that yield better accuracy. It also has the second-fastest elapsed time. In other words, its computational complexity is lower than other algorithms except for the PSO algorithm.

\begin{table}[b!]
\begin{center}
%\begin{minipage}{174pt}
\caption{Comparison of MPSO and other metaheuristics - 500 iterations.}\label{table5}%
\begin{tabular}{@{}p{2cm} p{2.7cm} p{2.7cm} p{2cm}@{}}
\toprule
Algorithm & Best fitness function value & Mean of fitness function values & Elapsed time (minute)\\
\midrule
			MPSO & 0.1879 & 0.1949 & 65 \\
			PSO & 0.2366 & 0.3115 & 64 \\
			GA & 0.4554 & 0.7568 & 76\\
			ACO & 0.5780 & 0.7592 & 73 \\
			ICA & 0.4261 & 0.5234 & 68 \\
			GWO & 0.2001 & 0.2641 & 81\\
			BA & 0.9914 & 1.2729 & 75 \\
			DE & 0.4067 & 0.7036 & 74\\ 
\bottomrule
\end{tabular}
%\end{minipage}
\end{center}
\end{table}

\hyperref[Fig8]{Fig. \ref{Fig8}} shows the RMSE evolution of algorithms; the convergence speed of the MPSO algorithm is faster than other ones, and it converges to its optimal value (0.1879) in less than 50 iterations. Second to the MPSO algorithm is the GWO with the fitness function value of 0.2001 reached at iteration 430 (fitness function values are presented in \hyperref[table5]{Table \ref{table5}}). The highest fitness function value at iteration 500 belongs to the BA (0.9914), which indicates its incapability in finding the global optimum or at least a better local optimum. Apart from that, GA has the slowest convergence rate, reaching its optimal value at iteration 490.

\begin{figure}[b!]%
\centering
\includegraphics[width=\textwidth]{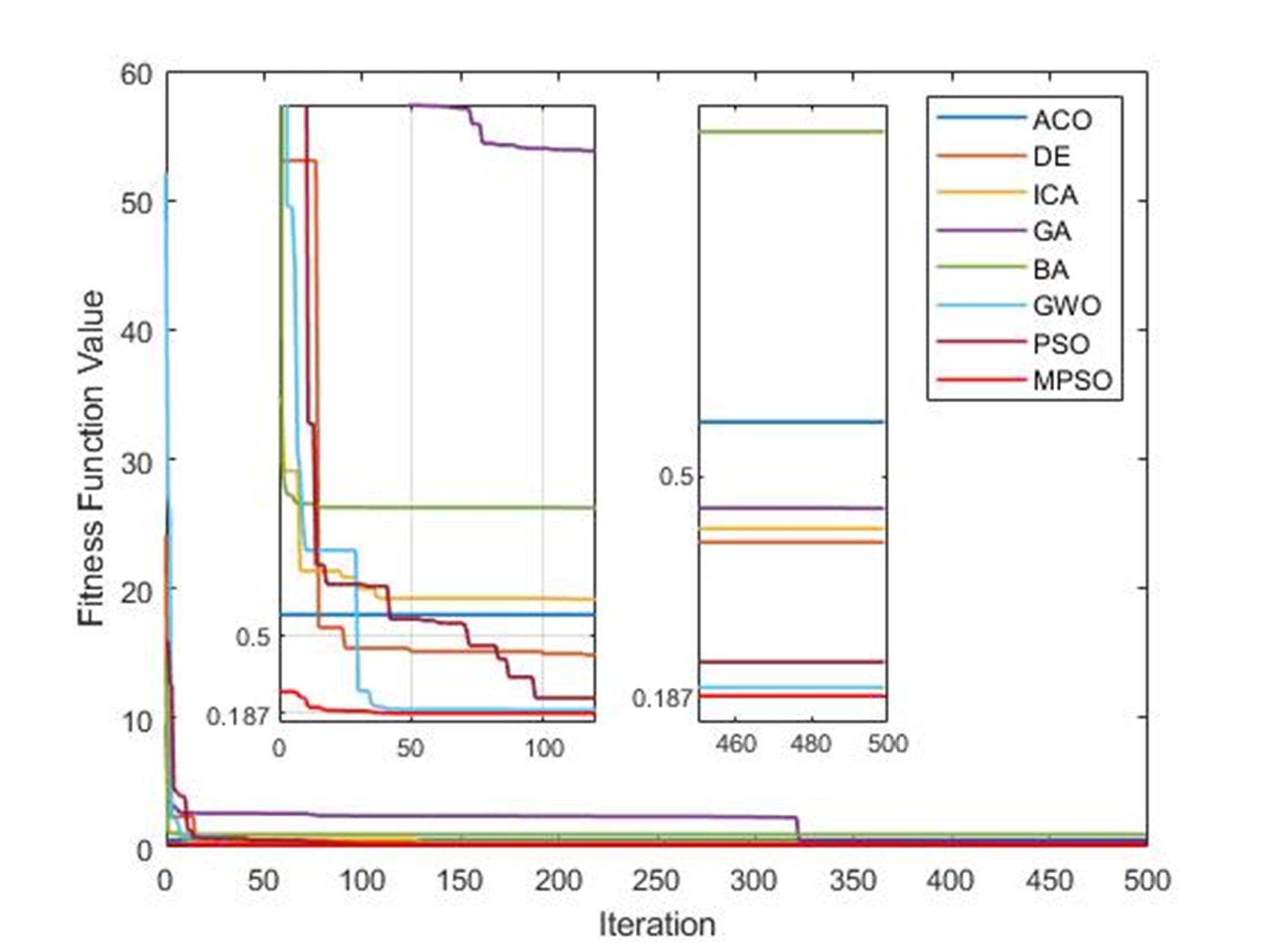}
\caption{RMSE evolution of the MPSO and other metaheuristics.}\label{Fig8}
\end{figure}

The algorithms are compared using Wilcoxon and Friedman tests to analyze the results statistically. \hyperref[table6]{Table \ref{table6}} and \hyperref[table7]{Table \ref{table7}} depict the results of the comparison. According to $p$-values in \hyperref[table6]{Table \ref{table6}}, the null hypothesis is rejected with a confidence level of 0.99, meaning MPSO shows a significant improvement over other algorithms. The high statistic numbers also verify that there is a significant difference between the performance of the algorithms. Moreover, the ranks computed through the Friedman test in \hyperref[table7]{Table \ref{table7}} show that MPSO is the best performing algorithm, while BA is the worst.

\begin{table}[t!]
\begin{center}
%\begin{minipage}{174pt}
\caption{Wilcoxon signed ranks test results.}\label{table6}%
\begin{tabular}{@{}p{3.5cm} p{2cm} p{1.3cm}@{}}
\toprule
Comparison & Statistic & $p$-value\\
\midrule
			MPSO vs PSO & 4129.0 & 0.00 \\
			MPSO vs GA & 3273.0 & 0.00 \\
			MPSO vs ACO & 8215.0 & 0.00 \\
			MPSO vs ICA & 7075.0 & 0.00 \\
			MPSO vs GWO & 8201.0 & 0.00 \\
			MPSO vs BA & 5535.0 & 0.00 \\
			MPSO vs DE & 18826.0 & 0.00 \\
\bottomrule
\end{tabular}
%\end{minipage}
\end{center}
\end{table}

\begin{table}[t!]
\begin{center}
%\begin{minipage}{174pt}
\caption{Friedman ranks.}\label{table7}%
\begin{tabular}{@{}p{2cm} p{0.9cm} p{0.9cm} p{0.9cm} p{0.9cm} p{0.9cm} p{0.9cm} p{0.9cm} p{0.9cm}@{}}
\toprule
  & MPSO & PSO & GA & ACO & ICA & GWO & BA & DE\\
\midrule
			Ranks & 1.000 & 3.312 & 7.256 & 5.930 & 5.116 & 2.218 & 7.268 & 3.900 \\ 
\bottomrule
\end{tabular}
\end{center}
\end{table}

Thus, we can conclude that the MPSO algorithm is an efficient optimization tool for the 3-DOF helicopter control and other control systems. Indeed, in a previous paper(\cite{naderi2020designing}), the MPSO algorithm has been implemented to design an interval type-2 fuzzy disturbance observer for a real ball and beam system.

\section{Conclusion}\label{sec6}
This paper proposed optimizing a fuzzy adaptive controller for the 3-DOF helicopter system through the MPSO algorithm. The controller has many parameters to be defined, including the membership functions' parameters of the fuzzy part and gains of the adaptive part. As the system is highly nonlinear, a slight change in the value of these parameters can significantly affect the controllers' performance. The MPSO algorithm has a high ability to avoid the local minimums, making it a suitable algorithm for optimizing the controller. We assessed the algorithm's performance by comparing it with the standard PSO algorithm and six other well-known metaheuristic algorithms. The results show the fast convergence rate and low computational complexity of the MPSO algorithm. In particular, the standard PSO algorithm does not achieve satisfactory results, particularly in the presence of uncertainties and disturbances. On the other hand, the controller optimized through the MPSO algorithm shows robustness properties to uncertainties and disturbance. Applying the MPSO algorithm to other control structures and further improving its performance can be considered in future research.

\bibliographystyle{plain}

\bibliography{Naderi}

\end{document}